%% file: writeup.tex
\documentclass[11pt]{article}

\usepackage{style_defaults}
\usepackage{amssymb}
\usepackage{macros}
\usepackage{geometry}
\usepackage{fullpage}
\usepackage{authblk}
\usepackage{booktabs}
\usepackage{makecell}
\usepackage{ragged2e}
\usepackage[round]{natbib}
\usepackage{hyperref}
\hypersetup{hidelinks}
\usepackage{xspace}
\usepackage{subcaption}
\usepackage{pgfplots}
\usepackage{tikz}
\usepgfplotslibrary{groupplots}
\usetikzlibrary{patterns,calc}
\pgfplotsset{compat=1.18}
\usetikzlibrary{arrows.meta, calc, decorations.pathreplacing}

\definecolor{llgray}{RGB}{245,245,245}
\definecolor{llblue}{RGB}{222,235,247}
\definecolor{llorange}{RGB}{252,232,206}
\definecolor{llgreen}{RGB}{226,242,218}

\graphicspath{{./}{figures/}}

\title{Smooth Partial Lotteries for Stable Randomized Selection}

\author{Alexander Goldberg}
\author{Giulia Fanti}
\author{Nihar B. Shah}
\affil{Carnegie Mellon University}
\affil{\{\texttt{akgoldbe,gfanti,nihars}\}\texttt{@andrew.cmu.edu}}
\date{}

\begin{document}

\maketitle

\begin{abstract}
Competitive selection processes, from scientific funding to admissions and hiring, use evaluations to score candidates, and eventually choose a subset of them based on those scores. Recently, many organizations have adopted \emph{partial lotteries}, which randomize selection based on evaluation scores. However, existing lottery designs are inherently unstable, as a small change to a single candidate's score can cause large shifts in their selection probabilities. This instability undermines a key goal of lotteries: reducing the influence of fine-grained score distinctions near the decision boundary. We propose \emph{smoothness} as a design principle for partial lotteries, formalizing it as a Lipschitz condition on the mapping from review scores over candidates to selection probabilities. We introduce the \emph{\ourlottery}, a simple mechanism in which selection probabilities scale linearly with estimated quality between an upper threshold, above which we always accept, and a lower threshold, below which we always reject. We prove that the \ourlottery's worst-case regret matches a lower bound for any smooth selection rule up to a factor of $(1 - \nselected/\napps)$, where $\nselected/\napps$ is the acceptance rate. We compare smooth selection to other stability notions like Individual Fairness and Differential Privacy, showing that the \ourlottery achieves a better smoothness--regret tradeoff than alternatives. Experiments on real peer review data from ICLR 2025, NeurIPS 2024, and the Swiss National Science Foundation demonstrate that existing lottery designs are highly unstable in practice even under perturbations to a single score. Our experiments also confirm the tightness of our theoretical analysis and show that our proposed \ourlottery achieves a better smoothness--utility tradeoff than alternatives in practice.
\end{abstract}

\section{Introduction}
\label{sec:intro}

\begin{figure}[t]
    \centering
    \resizebox{\linewidth}{!}{\input{figures/diagram1.tex}}
    \caption{\textbf{\ourlottery{} responds smoothly to score perturbations.}
    Marginal selection probabilities before and after a one-point increase in
    proposal~5's review score, with $2$ awards for $8$ candidates. Under the
    thresholded partial lottery, proposal~8 is selected deterministically, while the remaining award is allocated with equal probability among proposals~6 and~7; the score increase lifts proposal~5, splitting the award three ways at~$1/3$ each. A one-point change thus causes a discontinuous jump  (proposal~5 from $0$ to $1/3$) and abrupt drops for proposals~6 and~7. Under our \ourlottery{}, the same perturbation induces only small, smoothly varying changes in every proposal's selection probability.}    \label{fig:smooth_illustration}
\end{figure}
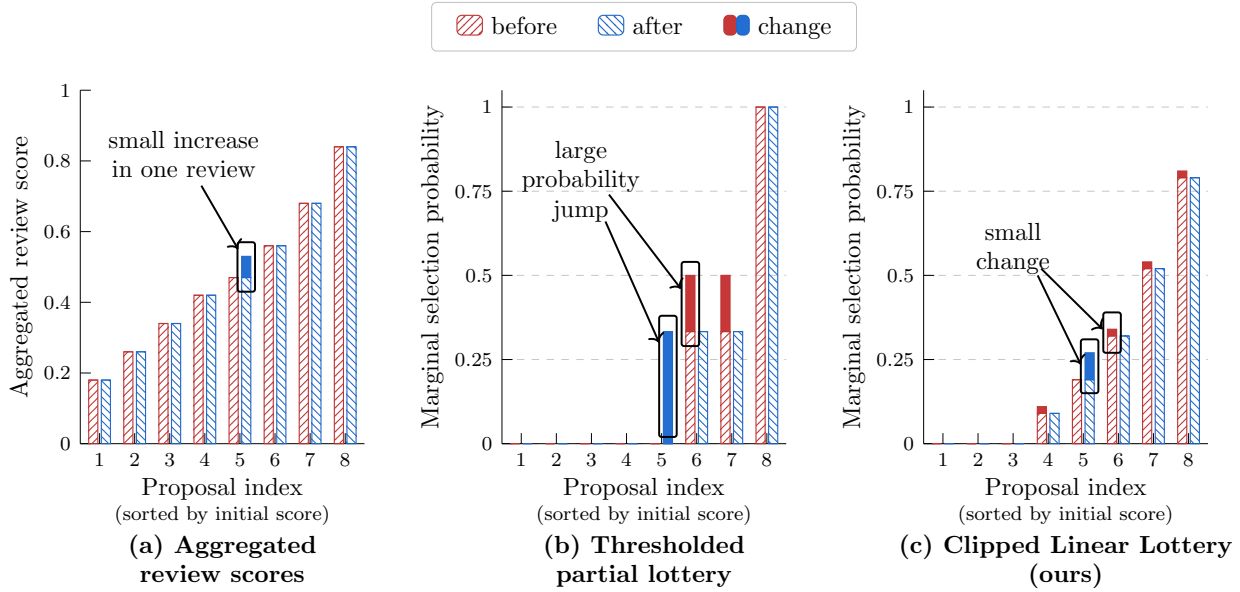

In settings like scientific funding, peer review, job screening, and university admissions, decision makers evaluate candidates and then select a subset for acceptance. Typically, multiple reviewers evaluate each candidate, these reviews are aggregated into a rank-ordered list, and then candidates are selected from the top down until the selection budget is exhausted. Recently, there has been growing interest in introducing \emph{randomization} into such selection processes. Most prominently, many scientific funding agencies have adopted ``partial lotteries'' to allocate research grants, including the New Zealand Health Research Council~\citep{liu2020acceptability}, the Swiss National Science Foundation~\citep{adam2019science}, the European Research Council~\citep{erc2023peerreview}, and many other funders~\citep{innovate2024, vwf2025lottery,britishacademy2025randomisation, fwf_1000_ideas, formas_explore_2025, luebber2025lottery}. 

In a partial lottery, a decision maker elicits quality evaluations and then selects candidates in a randomized fashion according to a probability distribution determined by those evaluations. A key practical motivation for randomization is to reduce deliberation over small quality differences:
\begin{quote}
\small\itshape
[Randomization] reduces the time needed for an assessment or moderation panel to
discuss applications. Panellists are less likely to spend time trying to arbitrarily differentiate between applications receiving similar quality-based scores that fall around the funding line, because they can ``agree to disagree'' by simply leaving it to chance~\citep{uk_metascience_2025}.
\end{quote}

Indeed, numerous funding agencies have cited reducing reviewer deliberation time as a justification when adopting partial randomization~\citep{liu2020acceptability, trentacosti2021lotteries, formas_explore_2025}.
Proponents of randomization have also argued that a partial lottery reduces arbitrariness near the funding line, encourages high-risk research, and mitigates reviewer partiality~\citep{fang2016research, horbach2022partial,gould2025threats,feliciani2024funding}.

In practice, existing lottery designs for scientific funding sort proposals into three groups: proposals above a high-quality threshold are funded automatically, proposals below a low-quality threshold are rejected, and proposals in the middle are entered into a uniform lottery with equal selection probabilities. This three-tier design is inherently unstable. As illustrated in Figure~\ref{fig:smooth_illustration}, a small increase in one review can move the highlighted candidate from rejection into the lottery group, causing a large jump in selection probability. More generally, any candidate near the boundary between reject, lottery, and accept groups may experience a discontinuous change in selection probability after an arbitrarily small score perturbation. In fact, \citet{goldberg2025principled} show that existing lottery designs can exhibit ``maximum instability,'' where an arbitrarily small change to a single review can shift the mechanism from deterministic selection to uniform random selection over all applicants. Our empirical results in this work show that this instability is not only theoretical. In Section~\ref{sec:exp_baselines}, we find that in real peer review processes at ICLR, NeurIPS, and the Swiss NSF, changing a single review by one point can change an applicant's selection probability by more than $0.3$ under existing lottery designs.

Instability is problematic because it reintroduces the kind of arbitrariness from small changes to review scores that randomization is meant to prevent. If two candidates have nearly identical evaluations, a small change to one review should not sharply change chances of selection. We capture this requirement through \emph{smoothness}: a selection rule is smooth if small changes in review scores can produce only bounded changes in selection probabilities.

In order to guarantee smoothness, we propose a \emph{\ourlottery}. Our \ourlottery rule first aggregates reviews into a score for each candidate. It then assigns selection probabilities that increase linearly with this score, with probabilities clipped at zero and one when candidates are sufficiently weak or strong. As illustrated in Figure~\ref{fig:smooth_illustration}, a small review perturbation moves the highlighted candidate only slightly along the linear part of the probability curve. The result is a small change in selection probability, rather than a discontinuous jump across a decision boundary. The \ourlottery has several practical benefits:
\begin{itemize}[leftmargin=*, itemsep=0.15em, topsep=0.25em]
    \item \emph{Stability.} We formally prove that small changes in individual reviews can only produce bounded changes in selection probabilities.
    \item \emph{Quality awareness.} Higher-scoring candidates receive higher selection probabilities.
    \item \emph{Transparency.} The rule is easy to explain, as selection probabilities increase linearly with estimated quality until they are clipped at zero or one.
    \item \emph{Fairness.} The intuitive notion of Individual Fairness~\citep{dwork2012fairness} requires that candidates with similar reviews receive similar selection probabilities. \citet{bairaktari_fair_cohort_selection} show that the clipped-linear form of the \ourlottery is optimal for trading off Individual Fairness and selection quality. This gives an additional reason to adopt the \ourlottery to ensure fairness in the treatment of similarly evaluated candidates.
\end{itemize}

These benefits require balancing smoothness against selection quality. A selection rule can be perfectly smooth by ignoring the reviews entirely, for example by sampling uniformly from all candidates. At the other extreme, a utility-maximizing rule deterministically selects the top-$\nselected$ candidates, but is not smooth. We formalize this tradeoff between smoothness and utility and show that the \ourlottery is near-optimal. 

Our \textbf{primary contributions} are:
\begin{enumerate}[label=(\arabic*), leftmargin=*, itemindent=0pt]
    \item \textbf{Formalizing smoothness for partial lotteries.} We define smooth selection rules using a Lipschitz condition on the mapping from review scores to marginal selection probabilities. This gives the interpretable guarantee that changing any single review score by one point shifts total selection probabilities by a bounded amount.

    \item \textbf{Design and analysis of the \ourlottery.} We introduce the \ourlottery and prove that it is $\smoothness$-smooth. We also prove a worst-case regret bound matching a lower bound for any smooth selection rule within a factor of up to a factor of $(1 - \nselected/\napps)$, where $\nselected/\napps$ is the acceptance rate.

    \item \textbf{Comparison to alternative stability notions.} We show that standard Differential Privacy (DP) does not imply smoothness, while the generalization of Metric DP gives smoothness but at a factor-$\nselected$ regret penalty in the worst case. We then analyze top-$\nselected$ softmax sampling and prove that it achieves smooth marginal selection probabilities, but incurs higher regret than the \ourlottery.

    \item \textbf{Empirical validation on real peer review data.} On three real peer review datasets (ICLR 2025, NeurIPS 2024, Swiss NSF) we demonstrate that existing partial lottery designs are highly unstable under single-review perturbations and show that the \ourlottery gives significantly better smoothness--utility tradeoffs than existing lottery designs in practice. We also demonstrate that the \ourlottery dominates the softmax rule across smoothness levels and empirically confirm tightness of our theoretical bounds.  
\end{enumerate}

An open-source implementation of our algorithm and all experimental code are available at
\href{https://github.com/akgoldberg/smooth_lottery}{\texttt{github.com/akgoldberg/smooth\_lottery}}.
\section{Related Work}
\label{sec:related_work}

\paragraph{Partial Lotteries for Scientific Funding and Peer Review.}
Recent works have designed algorithms for partial lotteries in scientific funding and peer review. MERIT~\citep{goldberg2025principled} and the Swiss NSF Partial Lottery~\citep{heyard2022rethinking} estimate \emph{intervals} for proposal quality, with overlapping intervals indicating uncertainty about the relative quality of proposals. These methods use intervals to decide when randomization is appropriate.

Our work targets a different design goal. Rather than explicitly modeling uncertainty about proposal quality, we require the mapping from review scores to marginal selection probabilities to be stable.  Empirically, we show that existing interval-based lotteries are highly non-smooth in practice; see Section~\ref{sec:exp_baselines} for details. A key difference between this prior work and our formulation is what information each approach extracts from review scores. In particular, MERIT~\citep{goldberg2025principled} emphasizes the use of ordinal information: intervals encode which proposals may or may not dominate others, making the approach natural when the numerical scale of review scores is not trusted or a funder wishes to avoid making assumptions about how reviewers interpret the scale. In contrast, our smoothness framework treats the review scale as meaningful enough to define small perturbations, and asks that these perturbations have only limited effects on marginal selection probabilities. In Section~\ref{sec:existing_lotteries}, we discuss how smoothness can be combined with interval-based validity constraints when a funder also wants to enforce them.

In practice, the choice depends on what concern is most important. If the funder wants the lottery to reflect uncertainty about which proposals are better, especially when numerical review scores are not fully trusted, then an interval-based method such as MERIT~\citep{goldberg2025principled} is a natural fit. If the funder wants to limit how much any one review score can affect selection probabilities, then \ourlottery is the more direct tool. These concerns can overlap, but the two mechanisms prioritize different failure modes.

\paragraph{Randomized Selection and the \ourlottery.}
The \ourlottery has appeared independently under different names in several contexts. In machine learning, \citet{kong2020_rankmax} introduced it as the \emph{RankMax} operator, an adaptive projection alternative to the softmax function, while \citet{martins2016sparsemax} proposed the closely related \emph{sparsemax} (for the $\nselected = 1$ case). Most relevant to our work, \citet{bairaktari_fair_cohort_selection} studied fair cohort selection and proved that the clipped-linear rule is exactly instance-optimal for maximizing linear utility subject to Individual Fairness constraints; we discuss this result in detail in Section~\ref{sec:individual_fairness}. However, the \ourlottery has not previously been proposed or analyzed as a mechanism for designing partial lotteries in evaluation and selection settings such as peer review, and it is not obvious \emph{a priori} that it is a good mechanism for our problem of interest. We provide a novel smoothness analysis---formalizing a Lipschitz condition on the mapping from review scores to selection probabilities---together with near-matching upper and lower bounds on the worst-case regret of any smooth selection rule. These results establish that the \ourlottery is near-optimal for a fundamentally different objective (the smoothness--regret tradeoff) than those considered in prior work, and that it achieves a more efficient smoothness--regret tradeoff than natural alternatives like the softmax.

\paragraph{Stable Selection.}
There is extensive literature on algorithmic stability notions for selection mechanisms, most prominently Differential Privacy (DP)~\citep{DworkMNS06, dwork2014algorithmic} and Individual Fairness (IF)~\citep{dwork2012fairness}. We compare our smoothness definition and the \ourlottery to these alternatives in Section~\ref{sec:alt_definitions}, showing that standard DP does not imply smooth marginals, that Metric DP implies smooth marginals but at a factor-$\nselected$ regret penalty~\citep{steinke2017interactive, bafna2017price}, and that IF and smoothness are formally incomparable. As part of this comparison, we give a novel analysis of the Lipschitz constant of the top-$\nselected$ softmax mechanism~\citep{mcsherry2007mechanism, kool2019stochastic} (known as the Exponential Mechanism in the DP literature), extending recent bounds for the $\nselected = 1$ case~\citep{nair2025softmax} to the combinatorial setting of $\nselected > 1$, which may be of independent theoretical interest.

\paragraph{Randomization in other aspects of peer review.} There are a few other parts of the peer review process where explicit randomization is used in practice. Computer science conferences employ automated methods to assign reviewers to papers, and randomization~\citep{jecmen2020manipulation} is used in this assignment to mitigate problems of fraud via collusion rings~\citep{littman2021collusion} and identity thefts~\citep{shah2025identity}. The idea here is that no matter how a malicious reviewer and/or author games the system~\citep{hsieh2024vulnerability}, there is a bounded probability of the target reviewer getting assigned target paper. Some computer science conferences use randomization to select which of the accepted papers will be presented as oral presentations~\citep{usenix2025policy,sigmod2025cfp}. A consequence of interest is that randomization can allow for causal analysis of  policies adopted in the review process~\citep{saveski2023counterfactual}. 

\section{Problem Formulation}
\label{sec:problem_formulation}

Our methods apply to settings such as admissions, scientific peer review, and job screening, where decision makers estimate candidate quality and select the top candidates. For concreteness, throughout we will describe a \emph{funder} evaluating \emph{candidates}.

\paragraph{Randomized Selection Rule.}
The funder evaluates $\napps$ candidates and seeks to select $\nselected$ candidates of highest quality. Each candidate receives $\nreviews$ numeric reviews, with scores normalized to lie in the unit interval. The review data is therefore represented as a matrix $\reviewmatrix \in [0,1]^{\napps \times \nreviews}$, where $\reviewmatrix_{i,j}$ denotes the score assigned to candidate $i \in [\napps]$ by reviewer $j \in [\nreviews]$.\footnote{We use the standard notation $[\kappa]$ to denote the set $\{1,\ldots,\kappa\}$ for any positive integer $\kappa$.} For simplicity of exposition, we assume each candidate receives the same number of reviews; our formulation naturally extends to variable numbers of reviews (Appendix~\ref{app:variable_reviews}).

The funder randomizes decisions by specifying a \emph{randomized selection rule}
$
\selectionrule: [0,1]^{\napps \times \nreviews} \to \Delta(\ksubsets),
$
where $\ksubsets = \{S \subseteq [\napps] : |S| = \nselected\}$ is the set of all $\nselected$-size subsets of candidates, and $\Delta(\ksubsets)$ denotes the set of probability distributions over $\ksubsets$. Given review data $\reviewmatrix$, the rule $\selectionrule(\reviewmatrix)$ defines a distribution over size-$\nselected$ subsets of candidates. Since candidates experience decisions through individual acceptance probabilities, we focus on the induced \emph{marginal selection probabilities} $\selectionprob: [0,1]^{\napps \times \nreviews} \to [0,1]^{\napps}$, defined by
$
\selectionprob_i(\reviewmatrix)
=
\PrOp_{S \sim \selectionrule(\reviewmatrix)}[\,i \in S\,]$ for $i \in [\napps]$.

\paragraph{Smoothness.}
We formalize the requirement that the selection rule be stable to changes in review scores as a \emph{smoothness} property: small changes in review scores should lead to small changes in candidates' chances of acceptance. Since these chances are determined by the marginal probabilities, we define smoothness as a Lipschitz condition on $\selectionprob$:

\begin{definition}[$\smoothness$-smooth selection rule]
\label{def:smoothness}
For $\smoothness \ge 0$, the selection rule $\selectionrule$ is \emph{$\smoothness$-smooth} if its induced marginal probability function $\selectionprob$ satisfies
\[
\|\selectionprob(\reviewmatrix) - \selectionprob(\reviewmatrix')\|_1
\;\le\;
\smoothness \,\|\reviewmatrix - \reviewmatrix'\|_{1,1}
\quad
\text{for all }
\reviewmatrix, \reviewmatrix' \in [0,1]^{\napps \times \nreviews},
\]
where $\|A\|_{1,1} = \sum_{i,j} |A_{i,j}|$.
\end{definition}

This condition ensures the stability of candidates' selection chances with respect to review scores: if a single review changes by $\delta$, the total absolute change in selection probabilities is at most $\smoothness\,\delta$.

\paragraph{Utility.}
A selection rule should also select high-quality candidates. Hence, we introduce a utility function that depends on reviews. Given review data $\reviewmatrix$, the funder has utility vector $\util(\reviewmatrix) \in \R^{\napps}$, where $\util_i(\reviewmatrix)$ represents the utility of selecting candidate $i$. The funder's utility when selecting a set of candidates is the sum of the selected candidates' utilities. Since $\selectionrule$ induces marginal probabilities $\selectionprob$, its expected utility is $\selectionprob(\reviewmatrix)^\top \util(\reviewmatrix)$. The \emph{regret} of a selection rule is its utility loss relative to optimal deterministic selection,
$
\regret(\selectionrule; \reviewmatrix)
= \text{OPT}(\reviewmatrix) - \selectionprob(\reviewmatrix)^\top \util(\reviewmatrix),
$
where $\text{OPT}(\reviewmatrix)$ is the utility obtained by selecting the $\nselected$ candidates with highest utility. Our goal is to design a $\smoothness$-smooth selection rule that minimizes the worst-case regret, giving the minimax design problem:
\[
\begin{aligned}
\min_{\selectionrule: [0,1]^{\napps \times \nreviews} \to \Delta(\ksubsets)}
\quad &
\max_{\reviewmatrix \in [0,1]^{\napps \times \nreviews}}
\left\{
\text{OPT}(\reviewmatrix)
-
\selectionprob(\reviewmatrix)^\top \util(\reviewmatrix)
\right\} \\
\text{subject to}
\quad &
\|\selectionprob(\reviewmatrix)-\selectionprob(\reviewmatrix')\|_1
\le
\smoothness \|\reviewmatrix-\reviewmatrix'\|_{1,1},
\qquad
\forall \reviewmatrix,\reviewmatrix' \in [0,1]^{\napps \times \nreviews},
\end{aligned}
\]
where $\selectionprob$ denotes the marginal selection probabilities induced by $\selectionrule$.

\paragraph{Smooth Utility Function.}
We assume the funder evaluates candidates using a utility function that varies smoothly with the review scores. Specifically:
\begin{definition}[$\utillipschitz$-Lipschitz utility function]
\label{def:lipschitz_utility}
Utility function $\util: [0,1]^{\napps \times \nreviews} \to \R^{\napps}$ is \emph{$\utillipschitz$-Lipschitz} if for all review matrices $\reviewmatrix, \reviewmatrix' \in [0,1]^{\napps \times \nreviews}$, $\|\util(\reviewmatrix) - \util(\reviewmatrix')\|_1 \le \utillipschitz \,\|\reviewmatrix - \reviewmatrix'\|_{1,1}$.
\end{definition}
For example, the mean score $\util_i(\reviewmatrix) = \tfrac{1}{\nreviews}\sum_{j=1}^{\nreviews} \reviewmatrix_{i,j}$ is $\utillipschitz$-Lipschitz for $\utillipschitz=1/\nreviews$. The median, min, and max score are all $1$-Lipschitz. The Lipschitz utility function assumption is necessary. If the utility function were arbitrarily non-smooth, then a small change in review scores could produce a large change in utility, making it impossible for any smooth selection rule to achieve low worst-case regret with respect to this utility function.

\section{Optimal Smooth Lottery Design}
\label{sec:lottery_theory}

In Section~\ref{sec:linear_lottery}, we introduce a simple, efficient, and interpretable \emph{\ourlottery} and analyze its smoothness and regret; in Section~\ref{sec:lower_bound}, we prove a lower bound establishing optimality of the \ourlottery up to a factor of $\frac{1}{1 - \nselected/\napps}$. Since smoothness and regret depend only on the marginal probabilities, we design mappings from review data to selection probabilities; any sampling scheme implementing these marginals can be used for the final selection.

\subsection{\ourlottery}
\label{sec:linear_lottery}

We now introduce the \emph{\ourlottery}, a clipped-linear partial lottery described in Algorithm~\ref{alg:linear_lottery}. The selection rule assigns marginal selection probabilities by scaling candidate utilities, shifting them by a common intercept to satisfy the budget, and clipping the result to $[0,1]$. This partitions candidates into three tiers: an auto-accept pool ($\selectionprob_i=1$), an auto-reject pool ($\selectionprob_i=0$), and a lottery pool where selection probabilities scale linearly with estimated quality. The slope $\slope$ controls the smoothness--utility trade-off. Smaller $\smoothness$ gives a flatter slope, expanding the lottery pool toward uniform selection. Larger $\smoothness$ gives a steeper slope, shrinking the lottery pool toward deterministic top-$\nselected$ selection. Figure~\ref{fig:linear_lottery_example} illustrates the \ourlottery on a concrete example. 

\input{figures/linear_lottery_example}

\renewcommand{\algorithmicrequire}{\textbf{Input:}}

\begin{algorithm}[tb]
\small
\caption{\ourlottery~\citep{kong2020_rankmax,bairaktari_fair_cohort_selection}}
\label{alg:linear_lottery}
\begin{algorithmic}[1]
\REQUIRE Reviews $\reviewmatrix$, utility $\util$ with Lipschitz constant $\utillipschitz$, smoothness parameter $\smoothness>0$, budget $\nselected$
\STATE Compute utilities $u_i \gets \util_i(\reviewmatrix)$ for all $i \in [\napps]$
\STATE Set the slope $\slope \gets \smoothness/(2\utillipschitz)$ 
\STATE Scale utilities $z_i \gets \slope \cdot u_i$
\STATE Compute intercept $\intercept \in \mathbb{R}$ satisfying budget constraint, 
\[
\sum_{i=1}^{\napps} \clip_{[0,1]}(z_i+\intercept)=\nselected
\]
where $\clip_{[0,1]}(x)=\min\{1,\max\{0,x\}\}$, using Algorithm~\ref{alg:linear_lottery_simple} (Appendix~\ref{app:linear_lottery_algo}).
\STATE Set marginal probabilities
\[
\selectionprob_i(\reviewmatrix)
\gets
\clip_{[0,1]}(z_i+\intercept),
\qquad i \in [\napps]
\]
\STATE Sample size-$\nselected$ set $S \subset [\napps]$ with marginal probabilities $\selectionprob(\reviewmatrix)$, e.g., using systematic sampling~\citep{madow1949theory}
\RETURN Marginal probabilities $\selectionprob(\reviewmatrix)$ and selected set $S$
\end{algorithmic}
\end{algorithm}

The \ourlottery has several useful properties.

\paragraph{Uniqueness.}
For every utility vector and budget, the \ourlottery produces a unique probability vector~\citep{wang2015projection}.

\paragraph{Efficient computation.}
There are two computationally nontrivial steps in Algorithm~\ref{alg:linear_lottery}. First, we must compute an intercept $\intercept$ satisfying the budget constraint (line 4). This can be done efficiently using existing algorithms for projection onto the capped simplex~\citep{wang2015projection,bairaktari_fair_cohort_selection}; for completeness, Algorithm~\ref{alg:linear_lottery_simple} in Appendix~\ref{app:linear_lottery_algo} gives one exact implementation. Second, the \ourlottery specifies marginal selection probabilities rather than a unique joint distribution over selected sets. Given marginals $\selectionprob$, any sampling scheme that returns size-$\nselected$ subsets with these marginals can be used for the final selection (line 6). One simple option is classical systematic sampling~\citep{madow1949theory}.

\paragraph{Monotonicity in budget.}
For fixed utilities and smoothness parameter $\smoothness$, increasing the budget $\nselected$ can never decrease any candidate's selection probability. This contrasts with existing lottery designs, which can violate this natural monotonicity property~\citep{goldberg2025principled}.

\paragraph{Monotonicity of the lottery pool.}
For fixed utilities and budget, the size of the active lottery pool $\{i : 0 < \selectionprob_i < 1\}$ is non-increasing in the smoothness parameter $\smoothness$. Equivalently, enforcing a smoother lottery by decreasing $\smoothness$ can only increase, or leave unchanged, the number of candidates subject to randomization. Thus, a funder that wants to cap the lottery-pool size, for example for political or institutional reasons, can sweep over $\smoothness$ and choose the smoothest lottery that satisfies this constraint.

Formal statements and proofs of the monotonicity properties appear in Appendix~\ref{app:linear_lottery_monotonicity}.

\subsubsection{Smoothness}

We first show that the \ourlottery satisfies the target smoothness guarantee. The following theorem shows that setting $\slope=\smoothness/(2\utillipschitz)$, as in Algorithm~\ref{alg:linear_lottery}, is sufficient to guarantee $\smoothness$-smoothness.

\begin{theorem}[Smoothness of the \ourlottery]
\label{thm:linear_smoothness}
If $\util$ is $\utillipschitz$-Lipschitz, then the \ourlottery in Algorithm~\ref{alg:linear_lottery} guarantees $\smoothness$-smoothness.
\end{theorem}

This smoothness bound is tight within a factor of $1 - \frac{1}{\napps}$; that is, there are instances where the \ourlottery is at least $(1 - \frac{1}{\napps})\smoothness$-smooth. We give the full proof in Appendix~\ref{app:proof_linear_smoothness} and empirically validate tightness in Section~\ref{sec:exp_smoothness_tightness}.

The proof intuition is as follows. If $\intercept$ were fixed, then changing the reviews by magnitude $\delta$ could change the scaled utilities by at most $\slope \utillipschitz \delta$, since $\util$ is $\utillipschitz$-Lipschitz and clipping is $1$-Lipschitz. The remaining factor of $2$ comes from enforcing the budget. The intercept $\intercept$ must shift so that $\sum_i \selectionprob_i(\reviewmatrix)=\nselected$, and this common shift can change many candidates' probabilities at once. In the worst case, this doubles the total $\ell_1$ change in marginal probabilities, giving at most $2\slope\utillipschitz\delta$. Thus choosing $\slope=\smoothness/(2\utillipschitz)$ guarantees $\smoothness$-smoothness.

\subsubsection{Regret}

We next bound the worst-case regret. Before doing so, we establish a geometric characterization of the \ourlottery that plays a central role in the analysis.

\begin{proposition}[Projection characterization~\citep{kong2020_rankmax}]
\label{prop:projection}
For every $\reviewmatrix$, the \ourlottery outputs the Euclidean projection of $\slope \cdot \,\util(\reviewmatrix)$ onto the capped simplex $\mathcal{C}_{\napps,\nselected} = \{\selectionprob \in [0,1]^{\napps} : \|\selectionprob\|_1 = \nselected\}$. That is,
$
\selectionprob(\reviewmatrix)
= \argmin_{p \in \mathcal{C}_{\napps,\nselected}} \tfrac{1}{2}\|p - \slope \cdot \,\util(\reviewmatrix)\|_2^2.
$
\end{proposition}

Because the \ourlottery finds the valid probability distribution that is geometrically closest to the scaled utilities, we can use the properties of this projection to bound regret. Our proof (Appendix~\ref{app:proof_linear_regret}) leverages this interpretation, algebraically rewriting the regret as a simple quadratic function that depends only on the chosen probabilities:

\begin{theorem}[Regret of the \ourlottery]
\label{thm:linear_regret}
The worst-case regret of the \ourlottery satisfies
\[
\max_{\reviewmatrix}
\regret(\selectionrule; \reviewmatrix)
\le
\frac{\nselected\!\left(1 - \frac{\nselected}{\napps}\right)\utillipschitz}
{2\,\smoothness}.
\]
\end{theorem}

\begin{corollary}[Mean utility]
\label{corr:linear_regret_mean_util}
If utilities are given by the mean review score for each candidate, then the worst-case regret of the \ourlottery satisfies
\[
\max_{\reviewmatrix}
\regret(\selectionrule; \reviewmatrix)
\le
\frac{\nselected\!\left(1 - \frac{\nselected}{\napps}\right)}
{2\,\smoothness \nreviews}.
\]
\end{corollary}

To interpret the bound, note that regret always lies in $[0,\nselected]$ because utilities lie in $[0,1]$. Uniform random selection among all candidates is perfectly smooth and has worst-case regret $\nselected(1-\nselected/\napps)$. Theorem~\ref{thm:linear_regret} bounds the regret of the \ourlottery by this natural baseline scaled by $\utillipschitz/(2\smoothness)$: larger $\smoothness$ permits a steeper rule and therefore yields a smaller regret bound. Corollary~\ref{corr:linear_regret_mean_util} follows immediately because the mean utility has Lipschitz constant $\utillipschitz = 1/\nreviews$. Intuitively, as the number of reviews per candidate increases, each individual review has less influence on the mean, so the worst-case regret bound decreases for any fixed smoothness parameter.

\subsection{Regret Lower Bound}
\label{sec:lower_bound}

We next show that the \ourlottery achieves nearly the lowest regret attainable by any smooth mechanism. For simplicity, we state the result for mean utility in the main text; however, the proof technique and near-optimality guarantee extend to many other utility functions, including the median, minimum, and maximum review score~(Appendix~\ref{app:general_lb}).

\begin{theorem}[Regret lower bound for $\smoothness$-smooth selection rules with mean utility]
\label{thm:lower-bound}
Let utilities be the mean review scores $\util_i(\reviewmatrix) = \tfrac{1}{\nreviews}\sum_{j=1}^{\nreviews}\reviewmatrix_{i,j}$, and let $\selectionrule$ be any $\smoothness$-smooth selection rule. Then, the worst-case regret of $\selectionrule$ is at least
\[
\max_{\reviewmatrix}
\regret(\selectionrule; \reviewmatrix)
\;\ge\;
\begin{cases}
\displaystyle
\frac{\nselected\!\left(1-\frac{\nselected}{\napps}\right)^2}
{2\,\nreviews\,\smoothness},
& \text{if }
\displaystyle
\smoothness \ge \frac{1}{\nreviews}\!\left(1-\frac{\nselected}{\napps}\right), \\[1.2em]
\displaystyle
\nselected\!\left(1-\frac{\nselected}{\napps}
-\frac{\smoothness\,\nreviews}{2}\right),
& \text{if }
\displaystyle
\smoothness < \frac{1}{\nreviews}\!\left(1-\frac{\nselected}{\napps}\right).
\end{cases}
\]
\end{theorem}

Theorem~\ref{thm:lower-bound} shows that the \ourlottery is near-optimal for the smoothness--regret tradeoff. In the regime
$\smoothness \ge \frac{1}{\nreviews}(1-\frac{\nselected}{\napps})$, the lower bound differs from the regret upper bound of the \ourlottery (Corollary~\ref{corr:linear_regret_mean_util}) only by a factor of $1-\frac{\nselected}{\napps}$. Thus, no smooth selection rule can substantially improve over the \ourlottery in worst-case regret. The regime $\smoothness \ge 1/\nreviews$ is especially natural because the mean review score itself is $(1/\nreviews)$-Lipschitz, so this corresponds to requiring selection probabilities to be roughly as smooth as the utility function. In the opposite limit, as $\smoothness \to 0$, the lower bound approaches $\nselected(1-\nselected/\napps)$, matching the regret of uniform random selection, whose probabilities are completely independent of the reviews.

The proof, given in Appendix~\ref{app:proof_regret_lower_bound}, uses an indistinguishability argument. Start from a baseline instance in which all candidates receive identical zero scores. Since the mechanism must select $\nselected$ candidates in expectation, some candidates must receive relatively low selection probability. We then perturb the scores of these least-likely candidates by a small amount $\delta$, making them the uniquely best candidates. Smoothness limits how much their selection probabilities can increase in response to this perturbation. As a result, any smooth rule must leave substantial probability mass on lower-quality candidates, incurring regret that scales with the strictness of the smoothness constraint.

Our argument extends to many row-wise utility functions beyond the mean, including the minimum, maximum, and median review score. Appendix~\ref{app:general_lb} gives the general version and shows that the \ourlottery remains near-optimal, up to the same factor of $1-\frac{\nselected}{\napps}$, for these utility functions as well.

\section{Alternative Approaches to Stable Selection}
\label{sec:alt_definitions}

We compare our smooth selection constraint against three prominent alternatives: Individual Fairness (IF), Differential Privacy (DP), and softmax-based randomized selection. As summarized in Table~\ref{tab:stability_comparison}, these perspectives lead to sharply different guarantees. IF and smooth marginals are formally incomparable, although the \ourlottery is also instance-optimal for IF (Section~\ref{sec:individual_fairness}). Standard DP does not imply $\ell_{1,1}$-smooth marginals because its neighboring relation treats arbitrarily small and large single review changes the same way (Section~\ref{sec:differential_privacy}). Metric DP fixes this discontinuity issue and implies smooth marginals, but using it to certify $\smoothness$-smoothness incurs a $\nselected\log(\napps/\nselected)$-fold regret penalty compared to the \ourlottery. Finally, top-$\nselected$ softmax achieves smooth marginals but with higher regret than the \ourlottery (Section~\ref{sec:softmax}).

\newcolumntype{L}[1]{>{\RaggedRight\arraybackslash}p{#1}}
\begin{table}[t]
\centering
\small
\begin{tabular}{@{} l L{0.31\linewidth} L{0.27\linewidth} L{0.24\linewidth} @{}}
\toprule
\textbf{} & \textbf{Definition} & \textbf{Implication} & \textbf{Regret} \\
\midrule
\makecell[tl]{\textbf{Smooth} \\ \textbf{selection} \\ (Def.~\ref{def:smoothness})}
& Marginal selection probabilities vary smoothly with review scores.
& No DP or IF guarantee.
& \ourlottery achieves near-optimal regret. \\
\midrule
\makecell[tl]{\textbf{Individual} \\ \textbf{fairness} (IF)}
& Candidates with similar utilities receive similar selection probabilities.
& No smooth marginals guarantee.
& \ourlottery is optimal. \\
\midrule
\makecell[tl]{\textbf{Differential} \\ \textbf{privacy} \\ ($\varepsilon$-DP)}
& Joint distribution over accepted candidates is stable under one-entry changes to the review matrix.
& Does not imply $\ell_{1,1}$-smooth marginals; arbitrarily small score changes may still cause discontinuous marginal changes.
& Not a direct certification of smoothness. \\
\midrule
\makecell[tl]{\textbf{Metric} \\ \textbf{DP}}
& Joint distribution changes with $\exp(\varepsilon\|\reviewmatrix-\reviewmatrix'\|_{1,1})$.
& Implies smooth marginals with $\smoothness \le \varepsilon\nselected$.
&  Regret worse by a factor of $\nselected\log(\napps/\nselected)$. \\
\bottomrule
\end{tabular}
\caption{Comparison of algorithmic stability definitions for selecting $\nselected$ of $\napps$ candidates.}
\label{tab:stability_comparison}
\end{table}

\subsection{Individual Fairness}
\label{sec:individual_fairness}

\emph{Individual Fairness} (IF)~\citep{dwork2012fairness} requires that candidates with similar utilities receive similar selection probabilities (formally: $|\selectionprob_i(\reviewmatrix) - \selectionprob_j(\reviewmatrix)| \le \alpha\,|\util_i(\reviewmatrix) - \util_j(\reviewmatrix)|$ for all $i,j$ on every $\reviewmatrix$). IF and our smooth-marginals condition are formally incomparable: smoothness bounds how the \emph{same} candidates' probabilities change across review matrices, whereas IF bounds how \emph{different} candidates' probabilities differ on the same matrix; neither implies the other (Appendix~\ref{app:if_separation}). Remarkably, prior work shows that the \ourlottery (with $\slope = \alpha$) is the exact, instance-optimal solution to the IF-constrained utility-maximization LP~\citep[Theorem 3]{bairaktari_fair_cohort_selection}. Thus, the \ourlottery is not only near-optimal for smoothness across datasets, but also instance-optimal for individual fairness on any fixed review matrix.

\subsection{Differential Privacy}
\label{sec:differential_privacy}

Perhaps the most widely studied notion of algorithmic stability is differential privacy (DP)~\citep{DworkMNS06}. DP ensures that an algorithm's output does not reveal too much about any single individual in the dataset. In our setting, DP constrains how much the \emph{joint} distribution over selected subsets can change when a single review score is perturbed.

\begin{definition}[$\varepsilon$-Differential Privacy]
\label{def:differential_privacy}
A selection rule $\selectionrule: [0,1]^{\napps \times \nreviews} \to \Delta(\ksubsets)$ satisfies $\varepsilon$-differential privacy ($\varepsilon$-DP) if for all review matrices $\reviewmatrix, \reviewmatrix' \in [0,1]^{\napps \times \nreviews}$ differing in at most one entry, and every subset $S \in \ksubsets$,
\[
\Pr[\selectionrule(\reviewmatrix)=S]
\le
e^\varepsilon
\Pr[\selectionrule(\reviewmatrix')=S].
\]
\end{definition}

Our $\smoothness$-smoothness condition is a statement about \emph{marginal} selection probabilities, and does not imply DP. A perfectly smooth selection rule can output uniform marginals $(\nselected/\napps,\ldots,\nselected/\napps)$ on all inputs, but implement those marginals using different joint distributions with disjoint support, which would not satisfy DP.

Conversely, standard DP does not imply our $\ell_{1,1}$-smoothness condition. The issue is that standard DP uses a discrete neighboring relation: any change to one review score is treated as one neighboring change, regardless of whether that score changes by an arbitrarily small amount or by a large amount. Thus, standard DP bounds the size of a change in the marginal probabilities, but this bound is constant and does not scale linearly with the magnitude of the change in review scores. Hence, in general, there need not exist a finite  $\smoothness$ such that the smoothness condition holds for arbitrarily small changes in a single entry of the review matrix $X$, as we show in the following proposition.

\begin{proposition}[Standard DP does not imply smooth marginals]
\label{prop:standard_dp_not_smooth}
For every $\varepsilon>0$, there exists an $\varepsilon$-DP selection rule that is not $\smoothness$-smooth with respect to $\|\cdot\|_{1,1}$ for any finite $\smoothness$.
\end{proposition}

We prove Proposition~\ref{prop:standard_dp_not_smooth} by counterexample in Appendix~\ref{app:dp_not_smooth}. The mechanism randomizes between two selected sets,
$
S_1=\{1,\ldots,\nselected\}$ and $S_2=\{\nselected+1,\ldots,2\nselected\}$,
by looking at only a single review score, $\reviewmatrix_{1,1}$, thresholding it at $1/2$, and using randomized response on the resulting bit to decide whether to favor $S_1$ or $S_2$. Since the probability ratio is bounded by $e^\varepsilon$, the rule is $\varepsilon$-DP under the standard one-entry replacement definition. But the marginals are discontinuous at the threshold: an arbitrarily small perturbation of $\reviewmatrix_{1,1}$ can change which set is favored, producing a constant-size change in the marginal probabilities. Hence standard DP does not imply smooth marginals.
To obtain a DP-style condition that rules out such discontinuities, the privacy definition must scale with the distance between inputs. We use metric differential privacy, which generalizes DP by replacing the discrete neighboring relation with an arbitrary metric on the input space~\citep{chatzikokolakis2013broadening}.

\begin{definition}[$\varepsilon$-Metric Differential Privacy]
\label{def:metric_differential_privacy}
A selection rule $\selectionrule: [0,1]^{\napps \times \nreviews} \to \Delta(\ksubsets)$ satisfies $\varepsilon$-metric differential privacy with respect to $\|\cdot\|_{1,1}$ if for all review matrices $\reviewmatrix,\reviewmatrix' \in [0,1]^{\napps \times \nreviews}$ and every subset $S \in \ksubsets$,
\[
\Pr[\selectionrule(\reviewmatrix)=S]
\le
\exp\left(
\varepsilon\|\reviewmatrix-\reviewmatrix'\|_{1,1}
\right)
\Pr[\selectionrule(\reviewmatrix')=S].
\]
\end{definition}

Metric DP controls the joint output distribution using the same distance on review matrices as our smoothness definition. This gives a bound on the smoothness of marginals:

\begin{proposition}[Metric DP implies smoothness of marginals]
\label{prop:dp_implications}
If a selection rule $\selectionrule$ satisfies $\varepsilon$-metric differential privacy with respect to $\|\cdot\|_{1,1}$, then its induced marginals $\selectionprob$ are $\smoothness$-smooth for some
\[
\smoothness \le \varepsilon\nselected.
\]
\end{proposition}

This bound is tight as $\varepsilon\|\reviewmatrix-\reviewmatrix'\|_{1,1}\to 0$; see Appendix~\ref{app:metric_dp_smooth}.

\paragraph{Metric DP is not regret-optimal for enforcing smooth marginals.}
Metric DP gives one way to certify $\smoothness$-smooth marginals: by Proposition~\ref{prop:dp_implications}, it suffices to take $\varepsilon \le \smoothness/\nselected$. Since metric DP implies standard DP for one-entry changes, known lower bounds for pure-DP top-$\nselected$ selection imply error $\Omega(\nselected \utillipschitz \log(\napps/\nselected)/\varepsilon)$~\citep{steinke2017interactive,bafna2017price}. Setting $\varepsilon \approx \smoothness/\nselected$ gives regret $\Omega(\nselected^2 \utillipschitz \log(\napps/\nselected)/\smoothness)$. This is larger than the lower bound for $\smoothness$-smooth rules in Theorem~\ref{thm:lower-bound} by an additional factor $\nselected\log(\napps/\nselected)$. Thus, metric DP does not by itself navigate the smoothness--regret tradeoff well.

\subsection{Softmax (Exponential Mechanism)}
\label{sec:softmax}

Section~\ref{sec:differential_privacy} shows that metric differential privacy does not, in itself, yield a good smoothness--utility tradeoff; however, Proposition~\ref{prop:dp_implications} is a worst-case guarantee and does not rule out that specific metric-differentially private mechanisms perform better than the generic smoothness bound in Proposition~\ref{prop:dp_implications}. We therefore study a well-known private selection mechanism: the \emph{Exponential Mechanism}~\citep{mcsherry2007mechanism}.

Given utilities $\util \in \mathbb{R}^{\napps}$ and temperature $\temp>0$, the softmax distribution is \[\softmax_{\temp}(\util)_i = \exp(\util_i/\temp)/\sum_{j}\exp(\util_j/\temp).\]
By the \emph{top-$\nselected$ softmax rule} we mean the procedure that samples $\nselected$ items \emph{without replacement} by repeatedly sampling from the re-normalized softmax distribution over the remaining items. Equivalently, this is the Gumbel-top-$\nselected$ procedure obtained by adding i.i.d.\ Gumbel noise to utilities and selecting the $\nselected$ largest perturbed utilities~\citep{kool2019stochastic}. The exponential mechanism naturally extends to metric differential privacy: if the utility function is $\utillipschitz$-Lipschitz with respect to a given metric, then sampling one item with probability proportional to $\exp(\util_i(\reviewmatrix)/\temp)$ satisfies $(2\utillipschitz/\temp)$-metric DP~\citep{kamalaruban2020not}.

\begin{theorem}[Smoothness and Regret of the Top-$\nselected$ Softmax Rule]
\label{thm:softmax_properties}
Assume $\util$ is $\utillipschitz$-Lipschitz. The top-$\nselected$ softmax rule with temperature $\temp$ induces $\smoothness$-smooth marginal selection probabilities with $\smoothness \le \frac{2\utillipschitz}{e \temp}$, and the worst-case regret satisfies $\max_{\reviewmatrix}\regret(\selectionrule;\reviewmatrix) \le \nselected\,\temp\,\log \napps$.
\end{theorem}

Theorem~\ref{thm:softmax_properties} shows that top-$\nselected$ softmax, like the \ourlottery, can be tuned to satisfy any desired smoothness level. In particular, setting $\temp = \frac{2\utillipschitz}{e\smoothness}$
guarantees $\smoothness$-smooth marginals and gives regret at most $
\frac{2\utillipschitz \nselected \log \napps}{e\smoothness}$. This smoothness guarantee is stronger than what follows from the generic differential privacy argument above, since it does not lose an additional factor of $\nselected$. However, it is still worse than the \ourlottery's regret bound by a logarithmic factor in $\napps$: the \ourlottery scales as $O(\nselected/\smoothness)$, while top-$\nselected$ softmax scales as $O(\nselected\log(\napps)/\smoothness)$. Thus, the \ourlottery provides a better worst-case smoothness--regret tradeoff. This comparison is based on our theoretical upper bound on worst-case regret of the softmax. A sharper smoothness analysis for softmax could in principle narrow the gap, but our empirical results in Section~\ref{sec:exp_smoothness_tightness} find instances where the bound is tight to within $5\%$.

As an additional practical advantage, the \ourlottery has \emph{sparse} support---deterministically funding the strongest proposals and rejecting the weakest, restricting randomness to borderline cases---whereas the softmax rule has full support and assigns strictly positive inclusion probability to every candidate. 

The regret bound follows by extending standard guarantees for the Exponential Mechanism in the differential privacy literature~\citep{dwork2014algorithmic}. However, the smoothness bound requires a novel analysis of the marginal inclusion probabilities induced by the top-$\nselected$ softmax rule, given in Appendix~\ref{app:proof_softmax}.

\section{Existing partial lottery designs}
\label{sec:existing_lotteries}

The most common lottery design used in practice partitions candidates into three tiers based on a score or rank statistic: auto-accept above a high threshold, auto-reject below a low threshold, and uniform randomization among candidates in the middle tier~\citep{liu2020acceptability, heyard2022rethinking, vwf2025lottery, britishacademy2025randomisation}. As discussed in the Introduction and illustrated in Figure~\ref{fig:smooth_illustration}, these designs share a structural source of instability. Since tier membership is determined by sharp thresholds, an arbitrarily small perturbation to a single review can move a candidate across a tier boundary, changing their selection probability by a constant amount. They are therefore inherently non-smooth.

Recent lottery designs, including the Swiss NSF procedure~\citep{adam2019science, heyard2022rethinking} and MERIT~\citep{goldberg2025principled}, also take \emph{interval estimates of candidate quality} as input. For example, an interval might be formed from the minimum and maximum review scores assigned to a proposal. These intervals capture uncertainty about the relative quality of proposals: if two candidates' intervals overlap, the funder treats their ordering as uncertain; if one candidate's interval lies strictly above another's, the funder treats the higher-scoring candidate as clearly dominating the lower-scoring one. This idea is formalized through an \emph{ex post validity} constraint~\citep{goldberg2025principled}: if candidate $i$'s lower bound exceeds candidate $j$'s upper bound, then $j$ should not be selected unless $i$ is selected as well. This guarantee limits randomization to candidates who are plausibly comparable, preventing outcomes in which a clearly dominated candidate is selected while the candidate who dominates them is not.

Can we combine ex post validity with smoothness? In general, the two desiderata may conflict. Suppose all candidates have near-zero-width, non-overlapping intervals. Then ex post validity forces deterministic rank-based selection: the top $\nselected$ candidates must be selected, and the remaining candidates must be rejected. But an arbitrarily small perturbation could move the $\nselected$-th candidate below the $(\nselected+1)$-st candidate, causing their selection probability to jump from $1$ to $0$. This directly conflicts with smoothness, which requires selection probabilities to change gradually under small review perturbations.

Appendix~\ref{app:ex_post_valid_sampling} gives two ways to reconcile the \ourlottery with ex post validity, depending on which guarantee the funder wants to enforce exactly. If the funder prioritizes exact smoothness, then it can relax the dominance relation by using intervals wide enough relative to the smoothness scale. Under this condition, whenever one candidate clearly dominates another, the \ourlottery already assigns probability $1$ to the dominant candidate or probability $0$ to the dominated candidate, so ex post validity holds automatically. If the funder instead prioritizes exact ex post validity, then it can project the \ourlottery marginals onto the polytope of ex post valid marginal probabilities and sample only from valid sets. This preserves interval dominance in every realized outcome, but may weaken the global smoothness guarantee. Thus, smoothness and ex post validity are not diametrically opposed design goals, but combining them requires making explicit which guarantee is enforced exactly and which is relaxed.

\section{Empirical Comparison}
\label{sec:experiments}

We complement our theory with experiments on real peer review datasets and synthetic data.

\paragraph{Experimental setup.}
We evaluate the \ourlottery and top-$\nselected$ softmax at acceptance rates $\nselected/\napps \in \{10\%,33\%,50\%\}$, and compare against existing partial lottery designs. We use ICLR 2025~\citep{openreview_iclr2025} ($\napps=3{,}710$, $\nreviews_{\min}=3$, scores $1$--$10$), NeurIPS 2024~\citep{openreview_neurips2024} ($\napps=4{,}034$, $\nreviews_{\min}=3$, scores $1$--$10$), Swiss NSF~\citep{heyard2022rethinking} ($\napps=353$, $\nreviews_{\min}=5$, scores $1$--$6$), and synthetic Beta reviews ($\napps=200$, $\nreviews=5$), drawn i.i.d.\ from $\mathrm{Beta}(\alpha,\alpha)$ on $[0,1]$ and discretized to $10$ levels. For ICLR and NeurIPS, we restrict to accepted papers and simulate allocating oral presentations, a randomized-selection setting recently used at multiple computer science conferences~\citep{usenix2025policy,sigmod2025cfp}. Since an allocation rate of roughly $10\%$ matches historical oral-presentation rates at NeurIPS and ICLR, we present $\nselected/\napps=10\%$ results in the main text and defer $\nselected/\napps\in\{33\%,50\%\}$ results to Appendix~\ref{app:additional_experiments}.

\subsection{Smoothness of Existing Partial Lottery Designs}
\label{sec:exp_baselines}

We first test whether existing partial lottery designs are smooth in practice. We evaluate two interval-based mechanisms, MERIT~\citep{goldberg2025principled} and the Swiss NSF procedure~\citep{heyard2022rethinking}. For each mechanism and dataset, we search over all single-review one-tick perturbations and identify the perturbation that maximizes the $\ell_1$ change in marginal selection probabilities. We report the resulting empirical local smoothness, defined as $\|\Delta \selectionprob\|_1 / \|\Delta \reviewmatrix\|_{1,1}$.

Both mechanisms take point estimates and intervals of utility as input. For all datasets, we use the mean review score as the utility and construct intervals using ``leave-one-out intervals''~\citep{goldberg2025principled}, computing the range of possible mean scores obtained by leaving out one reviewer at a time.

\begin{figure}[t]
    \centering

    \begin{subfigure}[t]{0.75\linewidth}
        \centering
        \includegraphics[width=\linewidth]{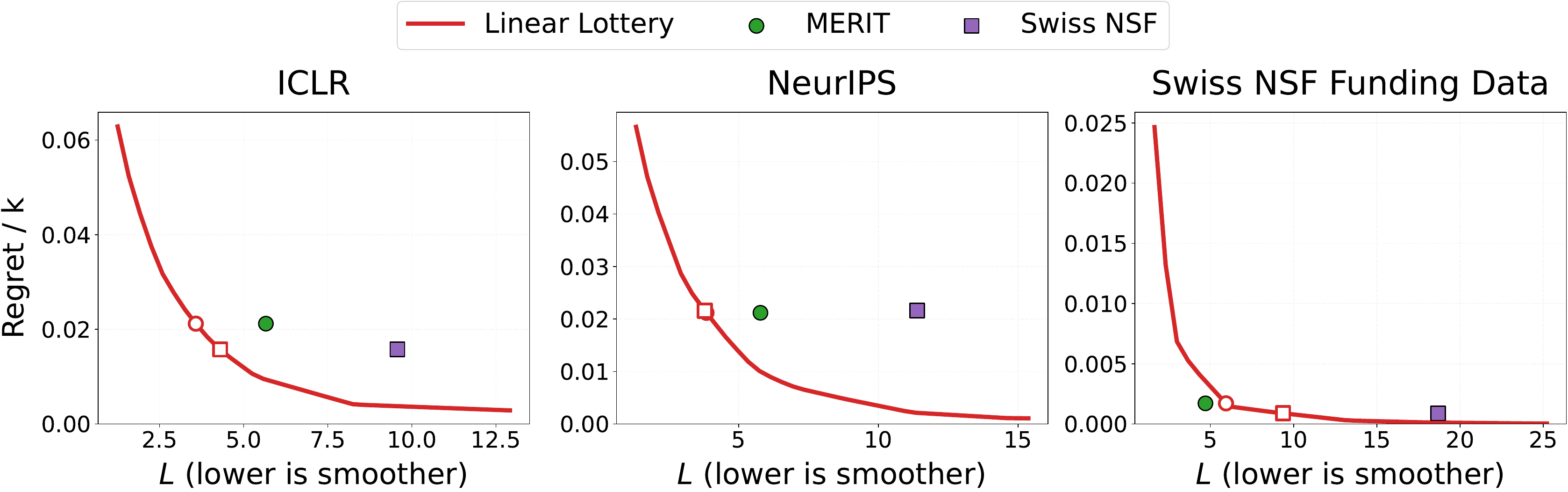}
        \caption{Acceptance rate $\nselected/\napps = 10\%$.}
        \label{fig:matched_regret_10}
    \end{subfigure}

    \vspace{0.75em}

    \begin{subfigure}[t]{0.75\linewidth}
        \centering
        \includegraphics[width=\linewidth]{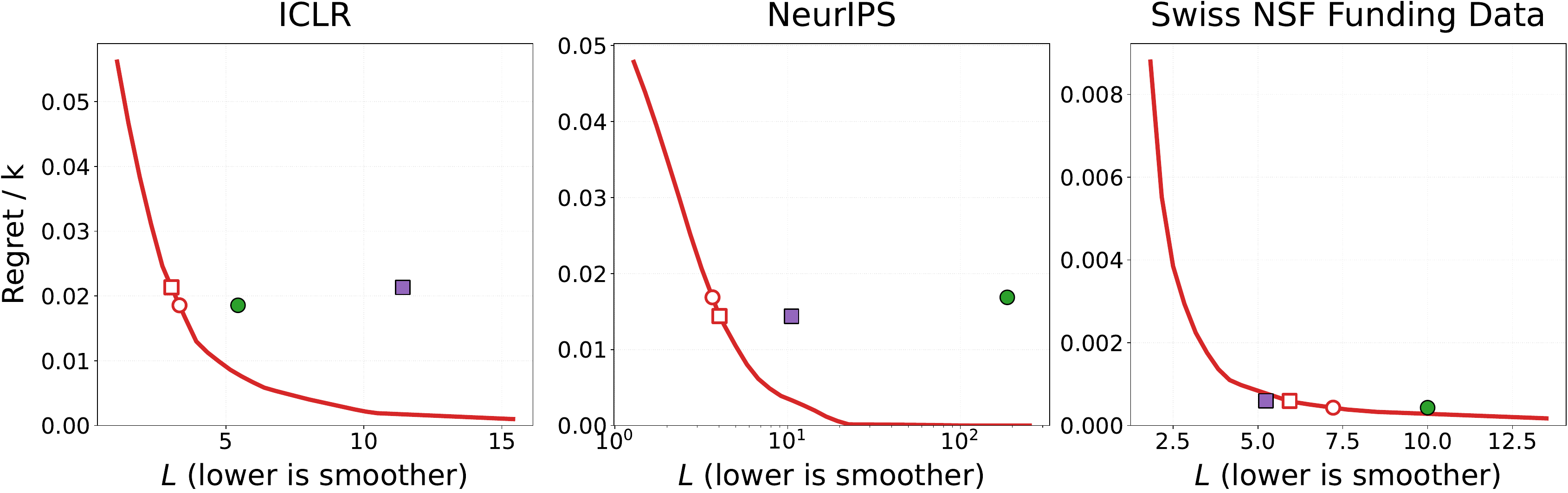}
        \caption{Acceptance rate $\nselected/\napps = 33\%$.}
        \label{fig:matched_regret_33}
    \end{subfigure}

    \caption{Regret--smoothness tradeoff for existing partial lottery mechanisms. Down and to the left indicates a better tradeoff. Points show MERIT and Swiss NSF at their empirical local smoothness under the worst one-review perturbation. Hollow red markers show worst-case smoothness guarantee of the \ourlottery at the same regret level.}
    \label{fig:baseline_local_sensitivity_existing}
\end{figure}

Figure~\ref{fig:baseline_local_sensitivity_existing} compares existing partial lottery designs with the \ourlottery. Existing mechanisms can be highly unstable under small perturbations. For example, on both NeurIPS and ICLR, a single one-point review change can change a paper's selection probability by more than $0.3$ under either MERIT or the Swiss NSF mechanism; full results appear in Appendix~\ref{app:additional_experiments}. By contrast, the \ourlottery achieves much better smoothness at the same level of regret.

\subsection{Regret--Smoothness Tradeoff}
\label{sec:exp_regret}

We next measure regret as a function of the smoothness parameter $\smoothness$. For each dataset and acceptance rate, we compute the regret of the \ourlottery and top-$\nselected$ softmax over a grid of target smoothness values. The \ourlottery's regret is computed exactly, while softmax regret is estimated using Monte Carlo sampling with $10{,}000$ samples per estimate.

\begin{figure}[tb]
    \centering
    \includegraphics[width=0.95\linewidth]{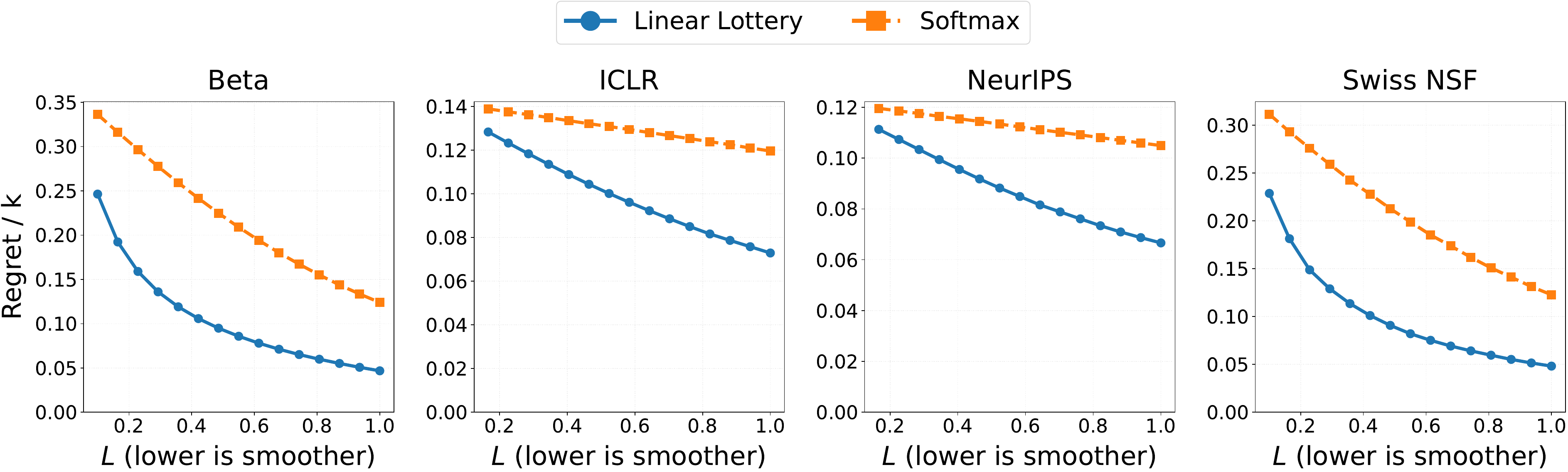}
    \caption{Regret vs.\ smoothness $\smoothness$ at acceptance rate $10\%$. Down and to the left indicates a better tradeoff. The \ourlottery achieves lower regret than softmax at every smoothness level across all datasets.}
    \label{fig:regret_vs_L}
\end{figure}

Figure~\ref{fig:regret_vs_L} shows the regret--smoothness tradeoff at acceptance rate $\nselected/\napps = 10\%$. We report regret normalized by the budget $\nselected$; under this normalization, the \ourlottery bound depends on the acceptance rate only through the factor $(1-\nselected/\napps)$, so the curves have similar shape at other acceptance rates. Appendix~\ref{app:additional_experiments} shows the same qualitative pattern for $\nselected/\napps \in \{33\%,50\%\}$. As predicted by theory, the \ourlottery consistently achieves lower normalized regret than top-$\nselected$ softmax across datasets and smoothness levels.

\subsection{Tightness of Smoothness Guarantees}
\label{sec:exp_smoothness_tightness}

Finally, we test whether our smoothness upper bounds are close to tight. For both the \ourlottery and top-$\nselected$ softmax, we construct near-worst-case utility profiles with $\util_i = 1$ for $i < \nselected$, $\util_\nselected = B \in [0,1]$, and $\util_i = 0$ for $i > \nselected$. We then perturb $\util_\nselected$ to $\util_\nselected \pm \eps$ and grid-search over $B$, $\eps$, and the perturbation direction to maximize the empirical ratio $\|\Delta \selectionprob\|_1 / \|\Delta \reviewmatrix\|_{1,1}$.

\begin{figure}[t]
    \centering
    \includegraphics[width=0.6\linewidth]{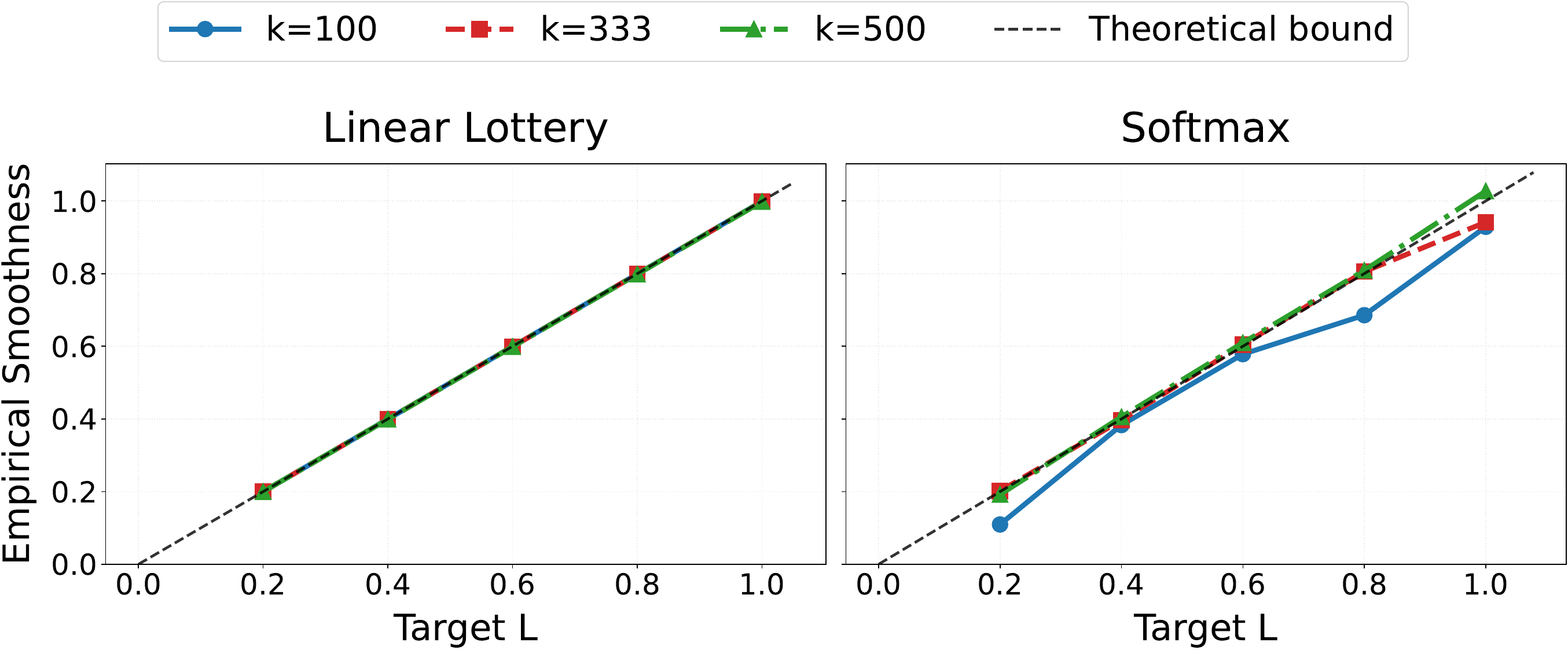}
    \caption{Worst-case empirical smoothness vs.\ target global smoothness $\smoothness$ for $\napps = 1000$ and $\nselected \in \{100,333,500\}$. The dashed line shows the theoretical upper bound. The bounds from Theorems~\ref{thm:linear_smoothness} and~\ref{thm:softmax_properties} are essentially tight for both mechanisms.}
    \label{fig:global_smoothness_ratio}
\end{figure}

Figure~\ref{fig:global_smoothness_ratio} compares the worst-case empirical smoothness to the target smoothness $\smoothness$. For the \ourlottery, the empirical lower bound is within $99\%$ of the theoretical upper bound at $\napps=1000$, confirming that Theorem~\ref{thm:linear_smoothness} is essentially tight. For top-$\nselected$ softmax, the empirical worst-case coefficient is within $95\%$ of the bound from Theorem~\ref{thm:softmax_properties}, showing that our softmax analysis is also close to tight.

\section{Discussion}
\label{sec:discussion}

We propose smoothness as a design principle for partial lotteries, requiring that small changes in review scores induce only small changes in selection probabilities. To satisfy this property, we propose the \ourlottery, analyze its smoothness and regret, and show that it is near-optimal for the resulting tradeoff up to a factor of $1-\frac{\nselected}{\napps}$. We compare smooth selection to Individual Fairness, Differential Privacy, and top-$\nselected$ softmax sampling, and show empirically on real peer review datasets that the \ourlottery achieves lower regret than softmax at matched smoothness, while existing partial lottery designs can be highly unstable under small review perturbations. Our results demonstrate that smoothness is a useful and tractable design principle for randomized selection in evaluation settings, and that the \ourlottery provides a simple, effective, and interpretable mechanism for enforcing it.

\paragraph{Limitations.}
First, our analysis assumes an exact, globally Lipschitz utility function, whereas practical utility estimates may be noisy, misspecified, or only locally smooth. Second, our regret bounds leave a factor of $(1-\nselected/\napps)$ gap between the \ourlottery upper bound and our lower bound for any smooth rule. Although this is a small constant factor in many selection settings, closing this gap would sharpen the optimality guarantees for smooth lotteries. Third, we do not study how smooth lotteries affect reviewer behavior or institutional decision-making in practice. An important direction for future work is to evaluate whether smoother mechanisms reduce fruitless deliberation near decision thresholds, and whether reviewers and applicants perceive them as more legitimate than threshold-based partial lotteries.

\paragraph{Future work.}

A natural next step is to study smooth lotteries in deployed evaluation settings, including whether they reduce deliberation near decision thresholds, improve perceived legitimacy among reviewers and applicants, or change reviewer behavior. Another direction is to extend the theory beyond exact, globally Lipschitz utilities to noisy, locally smooth, or learned utility functions. Ideas from differential privacy, such as calibrating noise to local or smooth sensitivity rather than worst-case global sensitivity, may be useful for designing instance-adaptive lotteries. 

\section*{Acknowledgments}

This work was supported in parts by NSF 1942124, 2200410, and DGE2140739 as well as the Gates Foundation.

\bibliographystyle{plainnat}
\bibliography{bibtex}

~\\
\appendix
\noindent{\bf \Large Appendices}

\section{Proofs}
\label{app:proofs}

\subsection{Smoothness of \ourlottery (Theorem~\ref{thm:linear_smoothness})}
\label{app:proof_linear_smoothness}

\begin{proof}
Let $u = \util(\reviewmatrix)$ and $u' = \util(\reviewmatrix')$. Let $\intercept$ and $\intercept'$ denote the intercepts for the \ourlottery on $\reviewmatrix$ and $\reviewmatrix'$, respectively. For each coordinate $i$, the selection probability is $\selectionprob_i = \clip_{[0,1]}(\slope u_i + \intercept)$. By the triangle inequality:
\[
|\selectionprob_i - \selectionprob'_i| \le \underbrace{|\clip_{[0,1]}(\slope u_i + \intercept) - \clip_{[0,1]}(\slope u'_i + \intercept)|}_{(A_i)} + \underbrace{|\clip_{[0,1]}(\slope u'_i + \intercept) - \clip_{[0,1]}(\slope u'_i + \intercept')|}_{(B_i)}.
\]
Since $\clip_{[0,1]}(\cdot)$ is $1$-Lipschitz, $A_i \le \slope|u_i - u'_i|$ and summing over $i$ yields $\sum_i A_i \le \slope \|u - u'\|_1$. For the second term, define $h_i(\beta) = \clip_{[0,1]}(\slope u'_i + \beta)$. Since $h_i(\beta)$ is non-decreasing in $\beta$ for all $i$, the differences $h_i(\intercept) - h_i(\intercept')$ share the same sign, so $\sum_i B_i = |\sum_i h_i(\intercept) - \sum_i h_i(\intercept')|$. Using the budget constraint $\sum_i \selectionprob_i = \sum_i \selectionprob'_i = \nselected$, we have $\sum_i h_i(\intercept') = \nselected = \sum_i \clip_{[0,1]}(\slope u_i + \intercept)$. Substituting:
\[
\sum_{i} B_i = \left| \sum_{i} \clip_{[0,1]}(\slope u'_i + \intercept) - \sum_{i} \clip_{[0,1]}(\slope u_i + \intercept) \right|.
\]
By the triangle inequality, this is bounded by $\sum_i A_i \leq \slope \|u - u'\|_1$. Hence, $\|\selectionprob - \selectionprob'\|_1 \le 2\slope \|u - u'\|_1$. Since $\util$ is $\utillipschitz$-Lipschitz and $\slope = \frac{\smoothness}{2\utillipschitz}$,
$\|\selectionprob - \selectionprob'\|_1 \le \smoothness \|\reviewmatrix - \reviewmatrix'\|_{1,1}.$

\paragraph{Tightness.} Consider $\util(\reviewmatrix) = (0,\dots,0)$ and $\util(\reviewmatrix') = (\eps, 0,\dots,0)$ for small $\eps>0$. The \ourlottery gives $\selectionprob(\reviewmatrix) = (\nselected/\napps,\dots,\nselected/\napps)$, and on perturbed utilities yields $\selectionprob_1(\reviewmatrix') = \nselected/\napps + \eps\slope\frac{\napps-1}{\napps}$ and $\selectionprob_i(\reviewmatrix') = \nselected/\napps - \eps\slope/\napps$ for $i>1$. Hence $\|\selectionprob(\reviewmatrix)-\selectionprob(\reviewmatrix')\|_1 = 2\slope\eps\frac{\napps-1}{\napps}$, demonstrating a realized $\smoothness\frac{\napps-1}{\napps}$ Lipschitz constant.
\end{proof}

\subsection{Regret of \ourlottery (Theorem~\ref{thm:linear_regret})}
\label{app:proof_linear_regret}

\begin{proof}
Let $u = \util(\reviewmatrix)$ and let $\selectionprob$ be the probability vector selected by the \ourlottery. Viewing the \ourlottery as Euclidean projection into $\mathcal{C}_{\napps, \nselected}$ (Proposition~\ref{prop:projection}), by projection optimality $(\selectionprob - \slope u)^\top (q - \selectionprob) \ge 0$ for all $q \in \mathcal{C}_{\napps, \nselected}$. Let $q^* \in \mathcal{C}_{\napps, \nselected}$ be the indicator vector for the top-$\nselected$ items. Substituting $q = q^*$ and rearranging yields $\slope u^\top (q^* - \selectionprob) \le \selectionprob^\top (q^* - \selectionprob)$. The left-hand side is exactly $\slope \cdot \regret(\selectionrule; \reviewmatrix)$. Thus
\begin{align}
\label{eqn:regret_obj}
\selectionprob^\top (q^* - \selectionprob) &= \sum_{i \in T} \selectionprob_i (1 - \selectionprob_i) - \sum_{i \notin T} \selectionprob_i^2 = \sum_{i \in T} \selectionprob_i - \sum_{i=1}^\napps \selectionprob_i^2,
\end{align}
where $T$ is the set of top-$\nselected$ candidates by utility. Maximizing \eqref{eqn:regret_obj} over $\selectionprob \in \mathcal{C}_{\napps, \nselected}$ is concave quadratic; by KKT conditions the optimal $\selectionprob_i = \frac{1}{2}(1 + \nselected/\napps)$ for $i \in T$ and $\selectionprob_i = \nselected/(2\napps)$ otherwise, giving $\selectionprob^\top (q^* - \selectionprob) \le \nselected(1 - \nselected/\napps)/4$. Therefore $\regret(\selectionrule; \reviewmatrix) \le \frac{\nselected(1 - \nselected/\napps)}{4 \slope}$ and taking $\slope = \frac{\smoothness}{2 \utillipschitz}$ completes the proof.
\end{proof}

\subsection{Proof of Regret Lower Bound (Theorem~\ref{thm:lower-bound})}
\label{app:proof_regret_lower_bound}

\begin{proof}
Let $\reviewmatrix^0$ be the all-zero matrix and $p^0=\selectionprob(\reviewmatrix^0)$. Let $T$ be the $\nselected$ candidates with smallest probabilities in $p^0$. Since $\sum_i p_i^0=\nselected$, we have $\sum_{i \in T} p^0_i \leq \nselected^2/\napps$. Fix $\delta\in[0,1]$ and define $\reviewmatrix^{(T)}$ by setting $\reviewmatrix^{(T)}_{i,j}=\delta$ for $i\in T$, $j\in[\nreviews]$, and $0$ otherwise; then $\|\reviewmatrix^{(T)}-\reviewmatrix^0\|_{1,1}=\nselected\nreviews\delta$. Let $p=\selectionprob(\reviewmatrix^{(T)})$. By smoothness, $\|p-p^0\|_1 \le \smoothness\nselected\nreviews\delta$. Since $\sum_i p_i=\sum_i p_i^0=\nselected$, $\sum_{i\in T}(p_i-p_i^0) \leq \tfrac{1}{2}\|p-p^0\|_1$, so $\sum_{i\in T}p_i \le \nselected^2/\napps + \smoothness\nselected\nreviews\delta/2$. On $\reviewmatrix^{(T)}$, $\util_i=\delta$ for $i\in T$ and $0$ otherwise, so $\text{OPT}=\nselected\delta$ and the rule attains utility $\delta\sum_{i\in T}p_i$. Thus
\[
\regret(\selectionrule;\reviewmatrix^{(T)}) \ge \nselected\delta\!\left(1-\tfrac{\nselected}{\napps}-\tfrac{\smoothness\nreviews\delta}{2}\right).
\]
Maximizing over $\delta\in[0,1]$ yields the claimed bound: $\delta^*=(1-\nselected/\napps)/(\smoothness\nreviews)$ if $\smoothness \ge \tfrac{1}{\nreviews}(1-\tfrac{\nselected}{\napps})$, else $\delta^*=1$.
\end{proof}

\subsection{Proof of Standard DP Does Not Imply Smooth Marginals (Section~\ref{sec:differential_privacy})}
\label{app:dp_not_smooth}

\begin{proposition}
\label{prop:dp_not_smooth}
For every $\varepsilon>0$, there exists an $\varepsilon$-DP selection rule whose induced marginal probabilities are not $\smoothness$-smooth for any finite $\smoothness$.
\end{proposition}

\begin{proof}
Fix any $1 \le \nselected < \napps$. Let
\[
S_1=\{1,\ldots,\nselected\},
\qquad
S_2=\{2,\ldots,\nselected+1\}.
\]
Thus $S_1$ and $S_2$ are both feasible selected sets of size $\nselected$, and they differ only by swapping candidate $1$ for candidate $\nselected+1$. Let the selection rule depend only on the review entry $\reviewmatrix_{1,1}$. Define
\[
a=\frac{e^\varepsilon}{1+e^\varepsilon}.
\]
If $\reviewmatrix_{1,1}<1/2$, the rule selects $S_1$ with probability $a$ and $S_2$ with probability $1-a$. If $\reviewmatrix_{1,1}\ge 1/2$, the rule swaps these probabilities, selecting $S_1$ with probability $1-a$ and $S_2$ with probability $a$.

This rule is $\varepsilon$-DP under the standard neighboring relation, as it is a version of randomized response. However, the induced marginal probabilities are discontinuous. For any $\delta>0$, let $\reviewmatrix,\reviewmatrix'$ satisfy
\[
\reviewmatrix_{1,1}=1/2-\delta,
\qquad
\reviewmatrix'_{1,1}=1/2+\delta,
\]
with all other entries equal. Then
\[
\|\reviewmatrix-\reviewmatrix'\|_{1,1}=2\delta.
\]
The only candidates whose marginal probabilities change are candidates $1$ and $\nselected+1$. Therefore,
\[
\begin{aligned}
\|\selectionprob(\reviewmatrix)-\selectionprob(\reviewmatrix')\|_1 = 
2\tanh(\varepsilon/2).
\end{aligned}
\]
The marginal change is constant in $\delta$, while the review distance tends to zero as $\delta\to 0$. Hence no finite $\smoothness$ can satisfy
\[
\|\selectionprob(\reviewmatrix)-\selectionprob(\reviewmatrix')\|_1
\le
\smoothness\|\reviewmatrix-\reviewmatrix'\|_{1,1}
\]
for all $\reviewmatrix,\reviewmatrix'$. Thus standard DP does not imply smooth marginals.
\end{proof}

\subsection{Proof of Metric DP Implies Smooth Marginals (Proposition~\ref{prop:dp_implications})}
\label{app:metric_dp_smooth}

\begin{proof}
Fix review matrices $\reviewmatrix,\reviewmatrix'$ and write
\[
d=\|\reviewmatrix-\reviewmatrix'\|_{1,1}.
\]
Let $q,q'$ be the distributions over $\ksubsets$ induced by $\selectionrule(\reviewmatrix)$ and $\selectionrule(\reviewmatrix')$, respectively. By $\varepsilon$-metric DP,
\[
q(S)\le e^{\varepsilon d}q'(S)
\qquad\text{and}\qquad
q'(S)\le e^{\varepsilon d}q(S)
\]
for every $S\in\ksubsets$. Thus the likelihood ratio between $q$ and $q'$ is bounded by $e^{\varepsilon d}$, which implies
\[
d_{\mathrm{TV}}(q,q')
\le
\frac{e^{\varepsilon d}-1}{e^{\varepsilon d}+1}
=
\tanh\!\left(\frac{\varepsilon d}{2}\right).
\]

We now relate total variation distance over selected sets to $\ell_1$ distance between marginal probabilities. For each candidate $i$,
\[
\selectionprob_i(\reviewmatrix)-\selectionprob_i(\reviewmatrix')
=
\sum_{S\ni i}\bigl(q(S)-q'(S)\bigr).
\]
Applying the triangle inequality and using $|S|=\nselected$ for every $S\in\ksubsets$,
\[
\begin{aligned}
\|\selectionprob(\reviewmatrix)-\selectionprob(\reviewmatrix')\|_1
&\le
\sum_{i=1}^{\napps}\sum_{S\ni i}|q(S)-q'(S)| \\
&=
\sum_{S\in\ksubsets}|S|\cdot |q(S)-q'(S)| \\
&=
\nselected\|q-q'\|_1 \\
&=
2\nselected\,d_{\mathrm{TV}}(q,q') \\
&\le
2\nselected
\tanh\!\left(\frac{\varepsilon d}{2}\right).
\end{aligned}
\]
Using $\tanh(x)\le x$ gives
\[
\|\selectionprob(\reviewmatrix)-\selectionprob(\reviewmatrix')\|_1
\le
\varepsilon\nselected d
=
\varepsilon\nselected
\|\reviewmatrix-\reviewmatrix'\|_{1,1}.
\]
Therefore the induced marginals are $\smoothness$-smooth with $\smoothness\le \varepsilon\nselected$.

The leading constant is tight as $\varepsilon d\to 0$. To see this, take two disjoint sets $S_1,S_2\in\ksubsets$ and consider two binary distributions supported only on $S_1$ and $S_2$ with likelihood ratio $e^{\varepsilon d}$:
\[
q(S_1)=\frac{e^{\varepsilon d}}{1+e^{\varepsilon d}},
\qquad
q'(S_1)=\frac{1}{1+e^{\varepsilon d}},
\]
and $q(S_2)=1-q(S_1)$, $q'(S_2)=1-q'(S_1)$. Then
\[
d_{\mathrm{TV}}(q,q')
=
\tanh\!\left(\frac{\varepsilon d}{2}\right).
\]
Since $S_1$ and $S_2$ are disjoint,
\[
\|\selectionprob(\reviewmatrix)-\selectionprob(\reviewmatrix')\|_1
=
2\nselected\,d_{\mathrm{TV}}(q,q')
=
2\nselected
\tanh\!\left(\frac{\varepsilon d}{2}\right).
\]
As $\varepsilon d\to 0$, this equals
\[
\varepsilon\nselected d + o(d),
\]
so the linear constant $\varepsilon\nselected$ cannot be improved in general.
\end{proof}

\subsection{Proof of Smoothness and Regret of Softmax (Theorem~\ref{thm:softmax_properties})}
\label{app:proof_softmax}

\begin{proof}[Proof of smoothness]
Let $u=\util(\reviewmatrix)$ and $u'=\util(\reviewmatrix')$, and $\selectionprob(u),\selectionprob(u')$ the corresponding marginals under top-$\nselected$ softmax with temperature $\temp$. We prove $v\mapsto\selectionprob(v)$ is $(2/(e\temp))$-Lipschitz w.r.t.\ utilities. By integrating the Jacobian along $v_t=(1-t)u+t u'$,
$
\|\selectionprob(u')-\selectionprob(u)\|_1 \le \big(\sup_t \|J(v_t)\|_{1\to 1}\big)\|u'-u\|_1.
$
Fix $j\in[\napps]$. With $G_1,\dots,G_\napps$ i.i.d.\ $\mathrm{Gumbel}(0,\temp)$ and $Y_i(v)=v_i+G_i$, holding noise fixed, increasing $v_j$ can only improve $j$'s rank and worsen others': $\partial_j\selectionprob_j\ge 0$ and $\partial_j\selectionprob_i\le 0$ for $i\neq j$. The fixed-budget identity $\sum_i\selectionprob_i(v)=\nselected$ gives $\sum_i \partial_j\selectionprob_i = 0$. Combining, the column sum equals $\sum_i|\partial_j\selectionprob_i| = 2\partial_j\selectionprob_j$.

Conditioning on $G_{-j}$, there is a deterministic threshold $\theta_j$ such that $j$ is selected iff $v_j+G_j\ge \theta_j$, namely the $\nselected$-th largest among $\{v_\ell+G_\ell:\ell\neq j\}$. Hence $\selectionprob_j(v) = \mathbb{E}_{G_{-j}}[1-F_G(\theta_j-v_j)]$. Differentiating and using $f_G(x)\le 1/(e\temp)$,
$
\partial_j\selectionprob_j = \mathbb{E}_{G_{-j}}[f_G(\theta_j-v_j)] \le 1/(e\temp).
$
Thus every column of $J$ has $\ell_1$ norm at most $2/(e\temp)$, so $\|\selectionprob(u)-\selectionprob(u')\|_1 \le \frac{2}{e\temp}\|u-u'\|_1$. Combined with $\|u-u'\|_1\le \utillipschitz\|\reviewmatrix-\reviewmatrix'\|_{1,1}$ gives the stated bound.
\end{proof}

\paragraph{Relation to prior softmax Lipschitz analyses.} Recent work establishes that the standard softmax map (the $\nselected=1$ case) is exactly $1/2$-Lipschitz with respect to $\ell_p$ norms via its closed-form Jacobian~\citep{nair2025softmax}. The Gumbel-top-$\nselected$ inclusion map does not possess a simple closed-form Jacobian. We overcome this by analyzing the marginal inclusion map directly: a shared-Gumbel coupling gives coordinate-wise monotonicity, the fixed-budget identity bounds the full Jacobian column sum via its diagonal, and a threshold representation combined with a Gumbel density bound controls the diagonal. Our $2/e\approx 0.736$ is within a small constant of the tight $1/2$ for $\nselected=1$; we conjecture our bound is essentially tight for general $\nselected$, and our experiments in Section~\ref{sec:exp_smoothness_tightness} support this.

\begin{proof}[Proof of regret bound]
Fix $u\in\mathbb{R}^{\napps}$ and consider the top-$\nselected$ softmax rule with temperature $\temp$ implemented as sequential sampling without replacement. Let $i_t$ denote the item selected at step $t$ and $R_t$ the remaining items, so $|R_t|=\napps-t+1$. The conditional one-step regret $M_t-\mathbb{E}[u_{i_t}\mid R_t]$, where $M_t=\max_{i\in R_t} u_i$, is bounded via the variational form of log-sum-exp: $\mathbb{E}[u_{i_t}\mid R_t] = \temp\log\sum_{i\in R_t} e^{u_i/\temp} - \temp H(p_t)$ where $p_t$ is the softmax on $R_t$. Since $\log\sum_{i\in R_t} e^{u_i/\temp}\ge M_t/\temp$ and $H(p_t)\le \log|R_t|$, $M_t-\mathbb{E}[u_{i_t}\mid R_t]\le \temp\log|R_t|$. Summing over $t$ and bounding $\log(\napps-t+1)\le\log\napps$ yields $\regret(\selectionrule;u)\le \nselected\temp\log\napps$.
\end{proof}

\subsection{Individual Fairness Details}
\label{app:if_separation}

\begin{definition}[$\alpha$-Individual Fairness]
\label{def:individual_fairness}
For a parameter $\alpha \ge 0$, a selection rule $\selectionrule$ is \emph{$\alpha$-individually fair ($\alpha$-IF)} if its induced marginal probabilities $\selectionprob$ satisfy
$
|\selectionprob_i(\reviewmatrix) - \selectionprob_j(\reviewmatrix)|
\le \alpha \,|\util_i(\reviewmatrix) - \util_j(\reviewmatrix)|
$ for all $\reviewmatrix \in [0,1]^{\napps \times \nreviews}$ and all $i,j \in [\napps]$.
\end{definition}

\paragraph{Optimality of the \ourlottery for IF.}
For any fixed review matrix $\reviewmatrix$, finding the utility-maximizing marginal probabilities that satisfy $\alpha$-IF can be written as the linear program
\begin{align}
    \label{IF_LP} \tag{IF-LP}
    \max_{\selectionprob \in [0,1]^{\napps}} \quad & \selectionprob^\top \util(\reviewmatrix) \\  \notag
    \text{s.t.} \quad & |\selectionprob_i - \selectionprob_j| \le \alpha \,|\util_i(\reviewmatrix) - \util_j(\reviewmatrix)| \quad \forall i,j \in [\napps], \\ \notag
    & \textstyle\sum_{i=1}^{\napps} \selectionprob_i = \nselected.
    \notag
\end{align}
Configuring the \ourlottery with $\slope = \alpha$ yields the exact, instance-optimal solution $\selectionprob^*$ to \eqref{IF_LP}~\citep[Theorem 3]{bairaktari_fair_cohort_selection}.

\paragraph{Smoothness $\nRightarrow$ IF.} A trivial selection rule that completely ignores review scores and always selects the first $\nselected$ candidates is perfectly stable ($0$-smooth across datasets), but can arbitrarily violate IF because candidates with identical utilities receive maximally different selection probabilities ($1$ vs.\ $0$).

\paragraph{IF $\nRightarrow$ Smoothness.} Because IF does not constrain behavior across datasets, an IF-compliant mechanism can be arbitrarily discontinuous. Consider a review matrix $\reviewmatrix$ yielding utilities $\util(\reviewmatrix) = (1, \dots, 1, 0, \dots, 0)$ with exactly $\nselected$ ones. A $1$-IF mechanism could output $\selectionprob(\reviewmatrix) = (1, \dots, 1, 0, \dots, 0)$. For a perturbed matrix $\reviewmatrix'$ where the first candidate's utility drops by $\delta$ so that $\util(\reviewmatrix') = (1-\delta, 1, \dots, 1, 0, \dots, 0)$, the mechanism could instead output a uniform lottery $\selectionprob(\reviewmatrix') = (\nselected/\napps, \dots, \nselected/\napps)$, which still satisfies IF. The change in the first candidate's selection probability is $|1 - \nselected/\napps|$, while the input perturbation is only $\delta$. As $\delta \to 0$, the Lipschitz constant must satisfy $\smoothness \ge |1 - \nselected/\napps|/\delta$, which diverges to infinity.

\section{Details of the \ourlottery}
\label{app:linear_lottery_details}

\subsection{Efficient Computation of the \ourlottery}
\label{app:linear_lottery_algo}

The \ourlottery guarantees a unique valid probability assignment~\citep{wang2015projection}. Because the total probability function $\totalprobfunc(b) = \sum_{i=1}^{\napps} \clip_{[0,1]}(z_i + b)$ is continuous and monotonically non-decreasing from $0$ to $\napps$, a valid intercept exists for any budget $\nselected \in [0, \napps]$. While multiple intercepts $\intercept$ might satisfy the budget if $\totalprobfunc(b)$ contains flat regions, these plateaus only occur when all candidate probabilities are saturated at $0$ or $1$, ensuring the resulting probability vector itself is strictly unique.

Implementing the \ourlottery requires finding this intercept $\intercept$ that satisfies the budget constraint. There are many computationally efficient implementations given in prior work, including binary search~\citep{kong2020_rankmax} and water-filling methods~\citep{bairaktari_fair_cohort_selection}. For completeness, Algorithm~\ref{alg:linear_lottery_simple} outlines an exact, non-iterative implementation with runtime $O(\napps^2)$, adapted from \citet{wang2015projection}. The algorithm first finds the $2\napps$ breakpoints at which $\totalprobfunc(b)$ changes (where a candidate's probability hits the floor of $0$ or the ceiling of $1$), then locates the interval of breakpoints $[\intercept_j, \intercept_{j+1}]$ that contains a budget-feasible intercept and computes the exact intercept in this interval.

\begin{algorithm}[H]
\caption{Simple Exact \ourlottery via Breakpoint Search}
\label{alg:linear_lottery_simple}
\begin{algorithmic}[1]
\REQUIRE Scaled utilities $z \in \mathbb{R}^{\napps}$ where $z_i = \slope \cdot \util_i$, budget $\nselected$
\STATE Initialize breakpoints where probabilities hit $0$ or $1$: \\
$B \gets \{ -z_i \mid i \in [\napps] \} \cup \{ 1 - z_i \mid i \in [\napps] \}$
\STATE Sort $B$ in ascending order: $\intercept_1 \le \intercept_2 \le \dots \le \intercept_{2\napps}$
\STATE Define the total probability function: $\totalprobfunc(b) = \sum_{i=1}^{\napps} \clip_{[0,1]}(z_i + b)$
\FOR{$j = 1$ \TO $2\napps - 1$}
    \STATE $\totalprobfunc_{\text{low}} \gets \totalprobfunc(\intercept_j)$, $\totalprobfunc_{\text{high}} \gets \totalprobfunc(\intercept_{j+1})$
   \IF{$\totalprobfunc_{\text{low}} \le \nselected \le \totalprobfunc_{\text{high}}$}
    \STATE \textbf{return} $\intercept_j$ \textbf{if} $\totalprobfunc_{\text{low}} = \totalprobfunc_{\text{high}}$ \textbf{else} $\intercept_j + (\intercept_{j+1} - \intercept_j) \frac{\nselected - \totalprobfunc_{\text{low}}}{\totalprobfunc_{\text{high}} - \totalprobfunc_{\text{low}}}$
\ENDIF
\ENDFOR
\end{algorithmic}
\end{algorithm}

\subsection{Monotonicity Properties of the \ourlottery}
\label{app:linear_lottery_monotonicity}

\begin{proposition}[Monotonicity Properties]
\label{prop:linear-lottery-properties}
The \ourlottery satisfies:
\begin{enumerate}
    \item \textbf{Monotonicity in budget:} For any fixed utilities $\util \in \R^{\napps}$ and slope $\slope > 0$, if $\selectionprob$ and $\selectionprob'$ are the solution vectors for budgets $\nselected$ and $\nselected'$ respectively, then $\nselected' > \nselected \implies \selectionprob'_i \ge \selectionprob_i$ for all $i \in [\napps]$.
    \item \textbf{Monotonicity of size of lottery pool:} For any fixed utilities $\util \in \R^{\napps}$ and budget $\nselected \in (0, \napps)$, the size of the active lottery pool, $R(\slope) = \{i \in [\napps] \mid 0 < \selectionprob_i < 1\}$, is monotonically non-increasing in $\slope$.
\end{enumerate}
\end{proposition}

\begin{proof}[Proof of (1)]
Let $\totalprobfunc(\intercept) = \sum_j \clip_{[0,1]}(\slope \util_j + \intercept)$. Each term is non-decreasing in $\intercept$, so $\totalprobfunc$ is non-decreasing. Let $\intercept_1,\intercept_2$ correspond to budgets $\nselected_1<\nselected_2$. If $\intercept_2<\intercept_1$ then $\nselected_2=\totalprobfunc(\intercept_2)\le\totalprobfunc(\intercept_1)=\nselected_1$, a contradiction. Thus $\intercept_2\ge\intercept_1$, and since $\clip_{[0,1]}$ is non-decreasing, $\selectionprob_i(\intercept_2)\ge\selectionprob_i(\intercept_1)$ for all $i$.
\end{proof}

\begin{proof}[Proof of (2)]
Reparameterize $\intercept=-\slope\tau$, so $\selectionprob_i = \clip_{[0,1]}(\slope(\util_i-\tau))$. Define $F(\slope,\tau) = \sum_i \clip_{[0,1]}(\slope(\util_i-\tau))$; for fixed $\slope$, $F$ is non-increasing in $\tau$, and for fixed $\tau$, non-decreasing in $\slope$. Let $\tau(\slope)$ satisfy $F(\slope,\tau(\slope))=\nselected$. If $\slope_2>\slope_1$, then for every $\tau$, $F(\slope_2,\tau)\ge F(\slope_1,\tau)$, hence $\tau(\slope_2)\ge\tau(\slope_1)$. Moreover $0<\selectionprob_i<1 \iff \tau<\util_i<\tau+1/\slope$, so $R(\slope) = \{i:\tau(\slope)<\util_i<\tau(\slope)+1/\slope\}$. Combining $\tau(\slope_2)\ge\tau(\slope_1)$ and $1/\slope_2\le 1/\slope_1$, $(\tau(\slope_2),\tau(\slope_2)+1/\slope_2) \subseteq (\tau(\slope_1),\tau(\slope_1)+1/\slope_1)$, so $R(\slope_2)\subseteq R(\slope_1)$.
\end{proof}

\section{Varying Number of Reviews Per Candidate}
\label{app:variable_reviews}

For simplicity, the main text assumes that every candidate receives the same number $\nreviews$ of reviews. This assumption is not essential. Suppose instead that candidate $i$ receives $\nreviews_i$ reviews, with $\nreviews_i \ge 1$, and let
\[
\nreviews_{\min} := \min_{i \in [\napps]} \nreviews_i .
\]
The review data is then a ragged array consisting of entries
$\reviewmatrix_{i,j} \in [0,1]$ for $j \in [\nreviews_i]$, and we define
\[
\|\reviewmatrix-\reviewmatrix'\|_{1,1}
=
\sum_{i=1}^{\napps}\sum_{j=1}^{\nreviews_i}
|\reviewmatrix_{i,j}-\reviewmatrix'_{i,j}|.
\]
All definitions in the main text extend directly with this metric: smoothness is still a Lipschitz condition on the induced marginal selection probabilities with respect to $\|\cdot\|_{1,1}$.

The theory also extends without change once the utility Lipschitz constant is updated. For example, if utility is the mean score
\[
\util_i(\reviewmatrix)
=
\frac{1}{\nreviews_i}
\sum_{j=1}^{\nreviews_i} \reviewmatrix_{i,j},
\]
then changing candidate $i$'s reviews by total amount $\delta$ changes $\util_i$ by at most $\delta/\nreviews_i$. Hence the mean-score utility is $\utillipschitz$-Lipschitz with
\[
\utillipschitz
=
\max_i \frac{1}{\nreviews_i}
=
\frac{1}{\nreviews_{\min}}.
\]
Thus, all upper-bound results for the \ourlottery continue to hold by replacing $\nreviews$ with $\nreviews_{\min}$ in the mean-score case. In particular, setting the \ourlottery slope to
\[
\slope = \frac{\smoothness}{2\utillipschitz}
=
\frac{\smoothness \nreviews_{\min}}{2}
\]
guarantees $\smoothness$-smoothness, and the regret bound becomes
\[
\max_{\reviewmatrix}
\regret(\selectionrule;\reviewmatrix)
\le
\frac{\nselected(1-\nselected/\napps)}
{2\nreviews_{\min}\smoothness}.
\]

This extension is conservative. Candidates with more than $\nreviews_{\min}$ reviews have mean utilities that are less sensitive to any single review perturbation, so the actual local smoothness of the \ourlottery may be better than the worst-case bound suggests. The lower-bound construction in the main text should also be interpreted with $\nreviews_{\min}$ as the worst-case review count: when candidates with the minimum number of reviews are available for the construction, the same indistinguishability argument applies verbatim. More generally, heterogeneous review counts can only make some candidates' utilities less sensitive to review perturbations, so using $\nreviews_{\min}$ gives a simple worst-case guarantee rather than a tight instance-specific characterization. We use this construction throughout our experiments where there is variation in the number of reviews per paper or proposal in real peer review datasets.

\section{Additional Theoretical Results}
\label{app:theory_results}

\subsection{General Regret Lower Bound}
\label{app:general_lb}

\begin{definition}[Lower Lipschitz Condition]
\label{def:c-gap}
Row-wise utility function $u:[0,1]^{\nreviews}\to[0,1]$ satisfies the \emph{$\utilgap$-lower Lipschitz condition} if $\inf_{\delta \in (0, 1/\utilgap]} \max_{x,y:\|x-y\|_1=\delta} \frac{|u(x)-u(y)|}{\|x-y\|_1} \ge \utilgap.$
\end{definition}

\begin{lemma}[Utilities satisfying the $\utilgap$-lower Lipschitz condition]
\label{lem:c-gap-table}
The following utility functions satisfy the $\utilgap$-lower Lipschitz condition with the specified constants. The table also lists the $\ell_1$-Lipschitz constant $\utillipschitz$ (Definition~\ref{def:lipschitz_utility}) of each utility function.
\end{lemma}

\begin{center}
\renewcommand{\arraystretch}{1.3}
\begin{tabular}{|c|c|c|c|c|}
\hline
Utility $\util$
& $\utillipschitz$
& $\utilgap$
& Baseline row $x^0(\delta)$
& Perturbed row $x(\delta)$ \\
\hline
Mean & $\frac{1}{\nreviews}$ & $\frac{1}{\nreviews}$ & $(0,\dots,0)$ & $(\frac{\delta}{\nreviews},\dots,\frac{\delta}{\nreviews})$ \\ \hline
Max & $1$ & $1$ & $(0,\dots,0)$ & $(\delta,0,\dots,0)$ \\ \hline
Min & $1$ & $1$ & $(0,\delta,\dots,\delta)$ & $(\delta,\delta,\dots,\delta)$ \\ \hline
Median (odd $\nreviews=2h+1$) & $1$ & $1$ & $(\underbrace{0,\dots,0}_{h+1}, \underbrace{\delta,\dots,\delta}_{h})$ & flip one $0$ to $\delta$ \\ \hline
\end{tabular}
\end{center}

\begin{remark}
For each of the above utilities, $\utilgap = \utillipschitz$. In general, $\utillipschitz$ is a worst-case global upper bound on the Lipschitz quotient over all row pairs, whereas $\utilgap$ is a guarantee on the maximum achievable quotient for any distance budget up to $1/\utilgap$. It always holds that $\utilgap \le \utillipschitz$, with equality when these worst-case witnesses achieve the global upper bound.
\end{remark}

\begin{theorem}[Generic regret lower bound]
\label{lem:generic-lb}
Let $\util$ be a row-wise utility function $\util_i(\reviewmatrix) = u(\reviewmatrix_{i,1},\dots,\reviewmatrix_{i,\nreviews})$ where $u$ satisfies the $\utilgap$-lower Lipschitz condition. Let $\selectionrule$ be any $\smoothness$-smooth selection rule. Then
\[
\max_{\reviewmatrix} \regret(\selectionrule; \reviewmatrix) \ge
\begin{cases}
\frac{\nselected\,\utilgap\!\left(1-\frac{\nselected}{\napps}\right)^2}
{2\,\smoothness},
& \text{if }
\smoothness \ge \utilgap\!\left(1-\frac{\nselected}{\napps}\right), \\[0.6em]
\nselected\!\left(1-\frac{\nselected}{\napps}
-\frac{\smoothness}{2\,\utilgap}\right),
& \text{if }
\smoothness < \utilgap\!\left(1-\frac{\nselected}{\napps}\right).
\end{cases}
\]
\end{theorem}

\begin{proof}
Fix $\delta\in(0, 1/\utilgap]$ and let $x(\delta), x^0(\delta)$ achieve the maximum in the $\utilgap$-lower Lipschitz condition, ordered so $u(x(\delta))\ge u(x^0(\delta))$. Define $\reviewmatrix^0$ with every row equal to $x^0(\delta)$, let $p^0=\selectionprob(\reviewmatrix^0)$, and let $T$ be the $\nselected$ candidates with smallest probabilities under $p^0$; then $\sum_{i\in T} p_i^0 \le \nselected^2/\napps$. Construct $\reviewmatrix^{(T)}$ by replacing rows in $T$ with $x(\delta)$. Then $\|\reviewmatrix^{(T)}-\reviewmatrix^0\|_{1,1}=\nselected\delta$ and by smoothness, $\|p-p^0\|_1 \le \smoothness\nselected\delta$, where $p=\selectionprob(\reviewmatrix^{(T)})$. Since $\sum_i p_i=\sum_i p_i^0=\nselected$, $\sum_{i\in T}p_i \le \nselected^2/\napps + \smoothness\nselected\delta/2$. On $\reviewmatrix^{(T)}$, each $i\in T$ gains at least $\utilgap\delta$ utility, so
$
\regret(\selectionrule;\reviewmatrix^{(T)}) \ge \nselected\utilgap\delta(1 - \nselected/\napps - \smoothness\delta/2).
$
Maximizing over $\delta\in(0, 1/\utilgap]$ yields the two-regime bound.
\end{proof}

\begin{remark}
Since $\utilgap = \utillipschitz$ for the mean, max, min, and median by Lemma~\ref{lem:c-gap-table}, the lower bound on the worst-case regret of any $\smoothness$-smooth selection rule is $\frac{\nselected (1-\nselected/\napps)^2 \utillipschitz}{2\smoothness}$ when $\smoothness \ge \utillipschitz(1-\nselected/\napps)$, matching the upper bound for the \ourlottery (Theorem~\ref{thm:linear_regret}) up to a factor of $1-\nselected/\napps$.
\end{remark}

\section{Ex Post Valid Sampling}
\label{app:ex_post_valid_sampling}

Several lottery designs take interval estimates of candidate quality as input rather than (or in addition to) point scores. The Swiss NSF procedure~\citep{adam2019science, heyard2022rethinking} sets the funding line using uncertainty intervals, and MERIT~\citep{goldberg2025principled} selects candidates to maximize worst-case utility over all rankings consistent with such intervals. A central requirement in interval-based designs is \emph{ex post validity}: every realized outcome should be consistent with the interval estimates, in the sense that no selected candidate is clearly dominated by an unselected one.

\begin{definition}[Ex post validity]
Given intervals $[\lb_i, \ub_i]$ for each candidate, define the dominance relation $i \succ j$ if $\lb_i > \ub_j$. A selected set $S$ is \emph{ex post valid} if $i \succ j$ and $j \in S$ imply $i \in S$.
\end{definition}

Smoothness and ex post validity may be incompatible: if many candidates have near-zero-width non-overlapping intervals, ex post validity demands a deterministic selection while smoothness requires that selection probabilities remain close. The \ourlottery specifies marginal selection probabilities $\selectionprob \in \mathcal{C}_{\napps,\nselected}$, but many joint distributions over size-$\nselected$ subsets can implement the same marginals; we give two resolutions to this tension depending on the funder's priorities.

\paragraph{Resolution 1: Core-width compatibility.}
If the intervals are sufficiently wide relative to the smoothness scale, the tension does not arise. Specifically, if
\begin{equation}
\label{eq:core-width}
\lb_i \le \util_i - \frac{1}{2\slope}
\quad\text{and}\quad
\ub_i \ge \util_i + \frac{1}{2\slope}
\qquad \forall i \in [\napps],
\end{equation}
then any dominance pair $i \succ j$ implies $\util_i - \util_j > 1/\slope$, and the \ourlottery automatically respects every dominance relation:

\begin{proposition}[Automatic ex post validity under core-width]
\label{prop:auto-validity}
Let $\selectionprob$ be the marginals produced by the \ourlottery with slope $\slope>0$. If~\eqref{eq:core-width} holds for all $i \in [\napps]$, then for every dominance pair $i \succ j$, $\selectionprob_i = 1$ or $\selectionprob_j = 0$. Consequently, every subset in the support of any sampling scheme implementing these marginals is ex post valid.
\end{proposition}

\begin{proof}
Suppose $i \succ j$, so $\lb_i > \ub_j$. By~\eqref{eq:core-width}, $\util_i-\tfrac{1}{2\slope} \ge \lb_i > \ub_j \ge \util_j+\tfrac{1}{2\slope}$, hence $\util_i-\util_j > 1/\slope$. Writing $\selectionprob_i=\clip_{[0,1]}(\slope(\util_i-\tau))$ for budget-enforcing threshold $\tau$, $\slope(\util_i-\tau) > \slope(\util_j-\tau)+1$. If $\selectionprob_j>0$, then $\slope(\util_j-\tau)>0$, so $\slope(\util_i-\tau)>1$ and $\selectionprob_i=1$.
\end{proof}

This resolution suggests that a funder could preserve exact smoothness by padding intervals to be wide enough relative to $1/\slope$ (equivalently, setting a stricter standard for asserting dominance). It effectively requires the funder to relax their ex post validity constraint.

\paragraph{Resolution 2: Projection onto valid marginals.}
When the core-width condition fails, a funder could sacrifice smoothness by projecting the \ourlottery marginals onto the set of all ex post valid marginal probabilities. Let $\mathcal{R}$ denote the family of ex post valid subsets and define the \emph{valid marginal polytope} $\validpolytope = \conv\{\mathbf{1}_S : S \in \mathcal{R}\}$. Given \ourlottery marginals $\selectionprob^{\mathrm{lin}}$, define the ex post feasible projection
\begin{equation}
\label{eq:projection_valid_marginals}
\hat{\selectionprob}
=
\argmin_{\selectionprob \in \validpolytope}
\|\selectionprob-\selectionprob^{\mathrm{lin}}\|_2^2.
\end{equation}
This quadratic program over $\validpolytope$ can be solved via Frank--Wolfe iteration, where the linear minimization oracle at each step requires finding a minimum-weight ex post valid subset, solvable in polynomial time over interval orders~\citep{goldberg2025principled}. After $T$ iterations, the iterate is a convex combination of at most $T$ valid subsets, yielding an explicit sampling distribution supported on ex post valid outcomes. This resolution preserves exact ex post validity but may sacrifice the global smoothness guarantee. 

\paragraph{Design choice.} Core-width compatibility preserves smoothness but requires wider intervals; projection preserves the original intervals and enforces strict ex post validity but may weaken smoothness. The choice depends on whether the funder prioritizes stable probabilities under score perturbations or strict interval-dominance in every realized outcome.

\section{Details on Datasets and Experimental Setup}
\label{app:data_details}

We use three real peer review datasets and one synthetic family:
\begin{itemize}
    \item \textbf{ICLR 2025} ($\napps = 3{,}710$, $\nreviews_{\min} = 3$): review scores on a $1$--$10$ scale.
    \item \textbf{NeurIPS 2024} ($\napps = 4{,}034$, $\nreviews_{\min} = 3$): review scores on a $1$--$10$ scale.
    \item \textbf{Swiss NSF} ($\napps = 353$, $\nreviews_{\min} = 5$): review scores on a $1$--$6$ scale.
    \item \textbf{Beta (synthetic)} ($\napps = 200$, $\nreviews = 5$): each review drawn i.i.d.\ from a symmetric $\text{Beta}(\alpha, \alpha)$ distribution on $[0,1]$, discretized to $10$ levels. In the main experiments we use $\alpha = 2$; in our Beta sweep we vary $\alpha \in \{1, 2, \ldots, 10\}$.
\end{itemize}
All review scores are normalized to $[0,1]$ using the known review scale, and utility is the mean normalized score. Unless otherwise noted, we set the softmax temperature to $\temp = 2\utillipschitz / (e\smoothness)$ so that its smoothness matches the \ourlottery's target $\smoothness$ (Theorem~\ref{thm:softmax_properties}). For comparisons against existing baselines, we use $\smoothness = 1/\nreviews$, matching the Lipschitz constant of the mean utility function.

\section{Additional Experimental Results}
\label{app:additional_experiments}

\begin{figure}[t]
    \centering
    \includegraphics[width=0.7\linewidth]{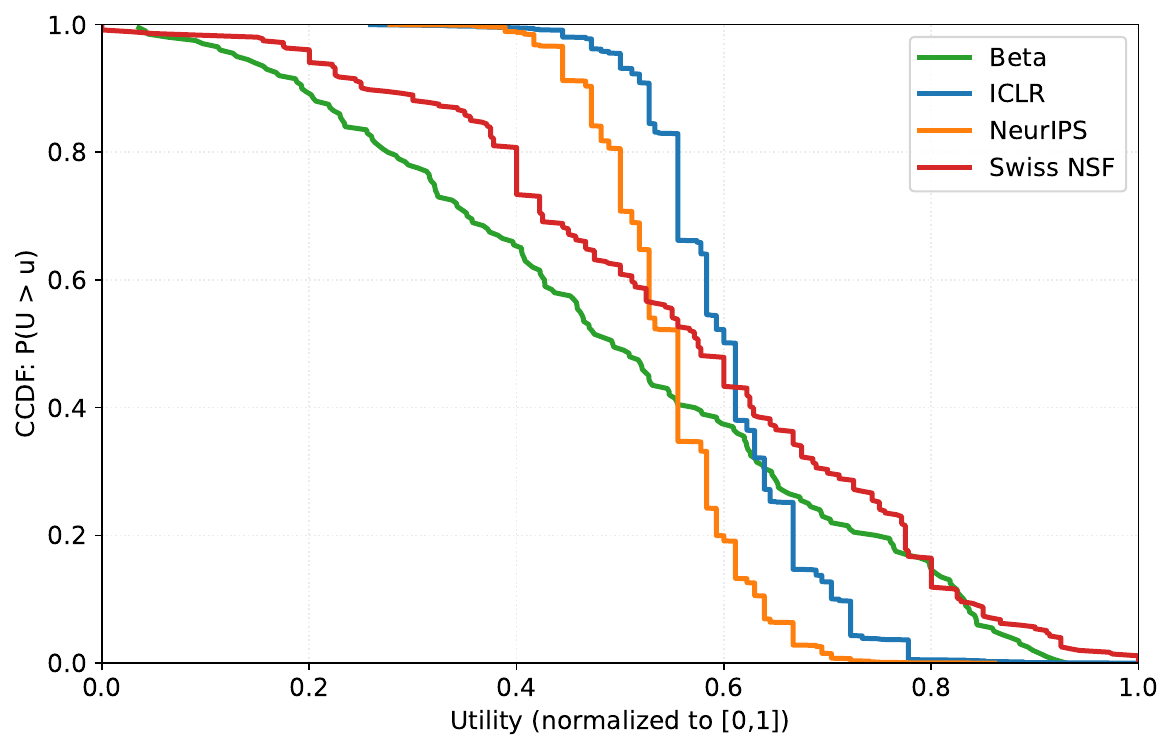}
    \caption{CCDF of mean normalized utilities across datasets. Swiss NSF and Beta have heavier tails with larger utility gaps, while ICLR and NeurIPS are concentrated near the center.}
    \label{fig:utility_ccdf}
\end{figure}

In this section, we provide additional experimental results. Figure~\ref{fig:utility_ccdf} shows the complementary cumulative distribution function (CCDF) of mean utilities across datasets. The ICLR and NeurIPS datasets have utility distributions tightly concentrated around $0.5$--$0.6$, with narrow interquartile ranges, while the Swiss NSF and Beta datasets spread utilities more broadly, creating larger gaps between candidates near typical acceptance thresholds. As we show below, this structure directly affects regret: datasets with more compressed utility distributions near the acceptance boundary incur higher regret under smooth selection rules, because the mechanism has less room to differentiate between candidates while satisfying the smoothness constraint.

\subsection{Regret-Smoothness Trade-off}

In this section, we provide additional results characterizing the tradeoff between regret and smoothness (Section~\ref{sec:exp_regret} in the main text).

\begin{figure}[t]
    \centering
    \includegraphics[width=0.7\linewidth]{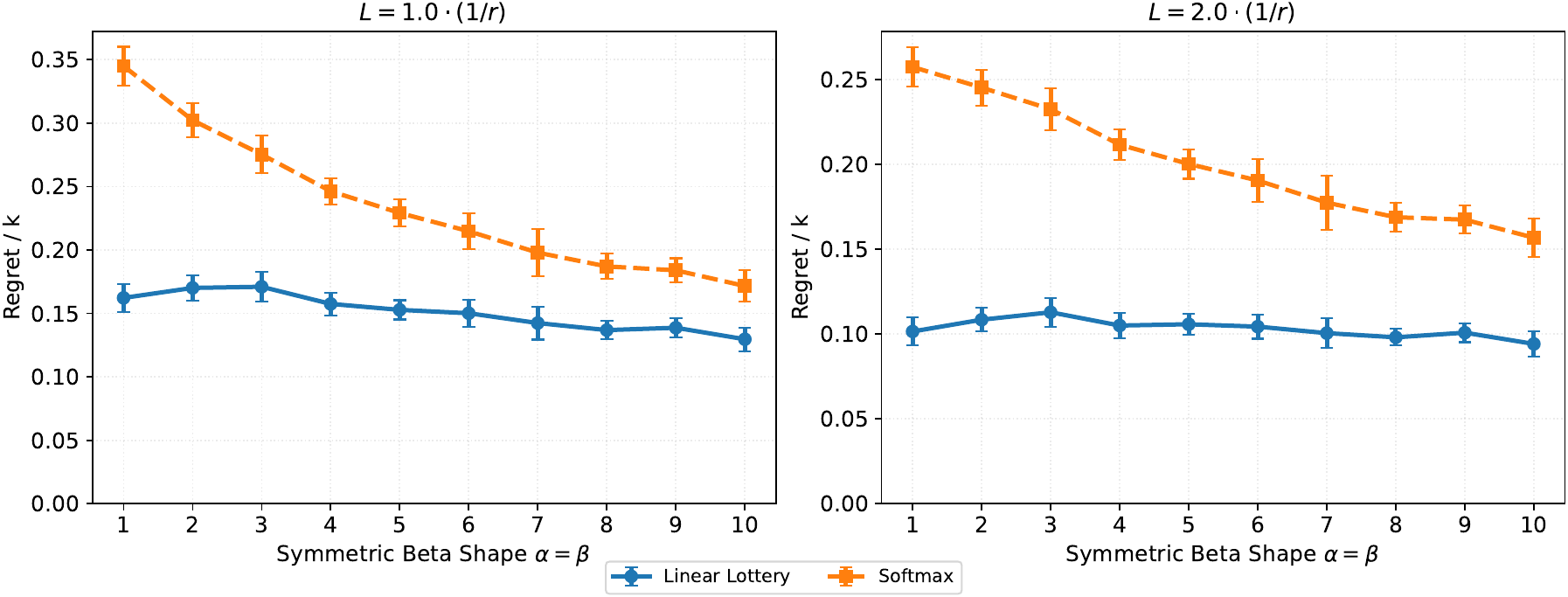}
    \caption{Regret vs.\ Beta shape parameter $\alpha = \beta$ at acceptance rate $10\%$. As the distribution concentrates around $0.5$ (larger $\alpha$), regret decreases for both selection rules because the worst-case utility gap narrows.}
    \label{fig:regret_beta_sweep}
\end{figure}

The Beta sweep (Figure~\ref{fig:regret_beta_sweep}) confirms this intuition that datasets with more compressed utility distributions near the acceptance boundary incur higher regret under smooth selection rules. As $\alpha = \beta$ increases, the Beta distribution concentrates around $0.5$, narrowing the range of realized utilities. Both the \ourlottery and softmax exhibit decreasing regret with increasing concentration. Across all shape parameters, the \ourlottery maintains its regret advantage over the softmax rule, consistent with the theoretical gap of $O(\log \napps)$.

Figures~\ref{fig:regret_smoothness_tradeoff_33} and~\ref{fig:regret_smoothness_tradeoff_50} show the tradeoff between regret and smoothness at higher acceptance rates. In both cases, the \ourlottery still outperforms softmax, although the gap is smaller. Intuitively, as the acceptance rate increases, regret normalized by the number of selected candidates $\nselected$ decreases because there are fewer rejected candidates that could have been selected instead, and randomization becomes less costly relative to the optimal top-$\nselected$ set. This is also reflected in our theoretical bounds, which scale with $(1-\frac{\nselected}{\napps})$. 

\begin{figure}
    \centering
    \includegraphics[width=0.8\linewidth]{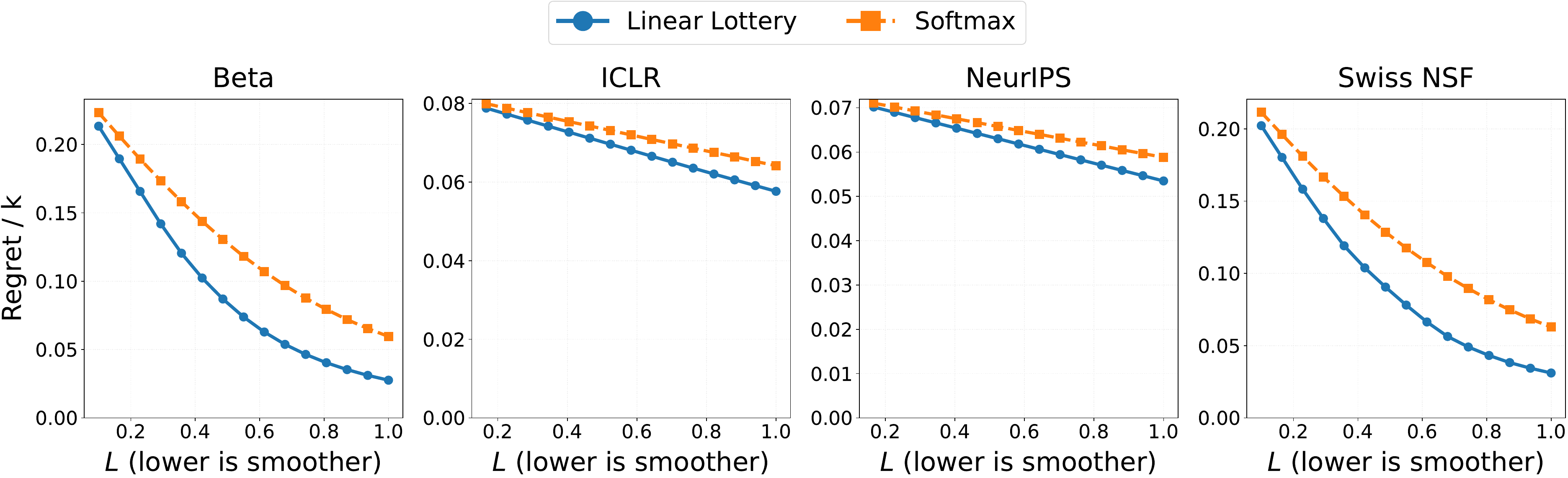}
    \caption{Regret vs.\ smoothness $\smoothness$ at acceptance rate $33\%$.}
    \label{fig:regret_smoothness_tradeoff_33}
\end{figure}

\begin{figure}[H]
    \centering
    \includegraphics[width=0.8\linewidth]{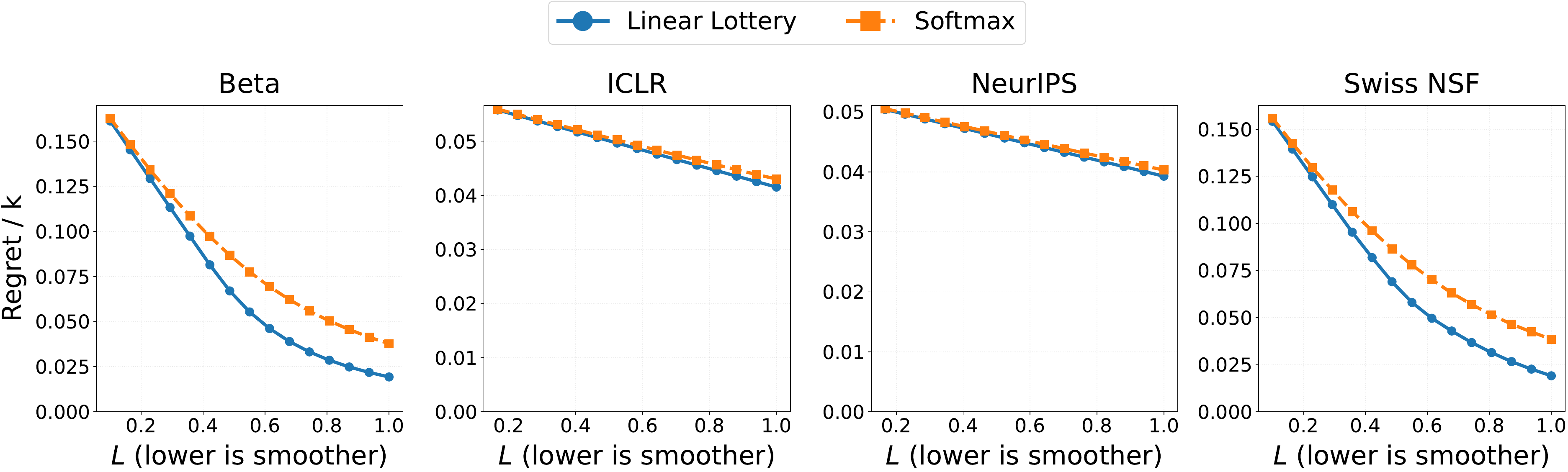}
    \caption{Regret vs.\ smoothness $\smoothness$ at acceptance rate $50\%$.}
    \label{fig:regret_smoothness_tradeoff_50}
\end{figure}

\subsection{Smoothness of Existing Partial Lottery Designs}

We now present additional results on the smoothness of two existing partial lottery mechanisms, MERIT and the Swiss NSF mechanism (Section~\ref{sec:exp_baselines} in the main text). Tables~\ref{tab:baseline_local_sensitivity_existing_10} and~\ref{tab:baseline_local_sensitivity_existing_3350} show that these mechanisms can be highly unstable under small review perturbations. Across datasets and acceptance rates, changing a single review score by one tick can induce large jumps in marginal selection probabilities for both mechanisms. This confirms that the instability of threshold-based partial lotteries is not just a worst-case theoretical concern, but appears empirically on real peer review and grant-funding data. Figure~\ref{fig:matched_regret_50} compares the \ourlottery to MERIT and the Swiss NSF mechanism at acceptance rates of $50\%$. The \ourlottery consistently gives a better regret--smoothness tradeoff, achieving substantially better smoothness at the same regret level.

\input{figures/baseline_local_sensitivity_table_existing.tex}

\input{figures/baseline_local_sensitivity_table_existing_appendix.tex}

\begin{figure}
    \centering
    \includegraphics[width=0.8\linewidth]{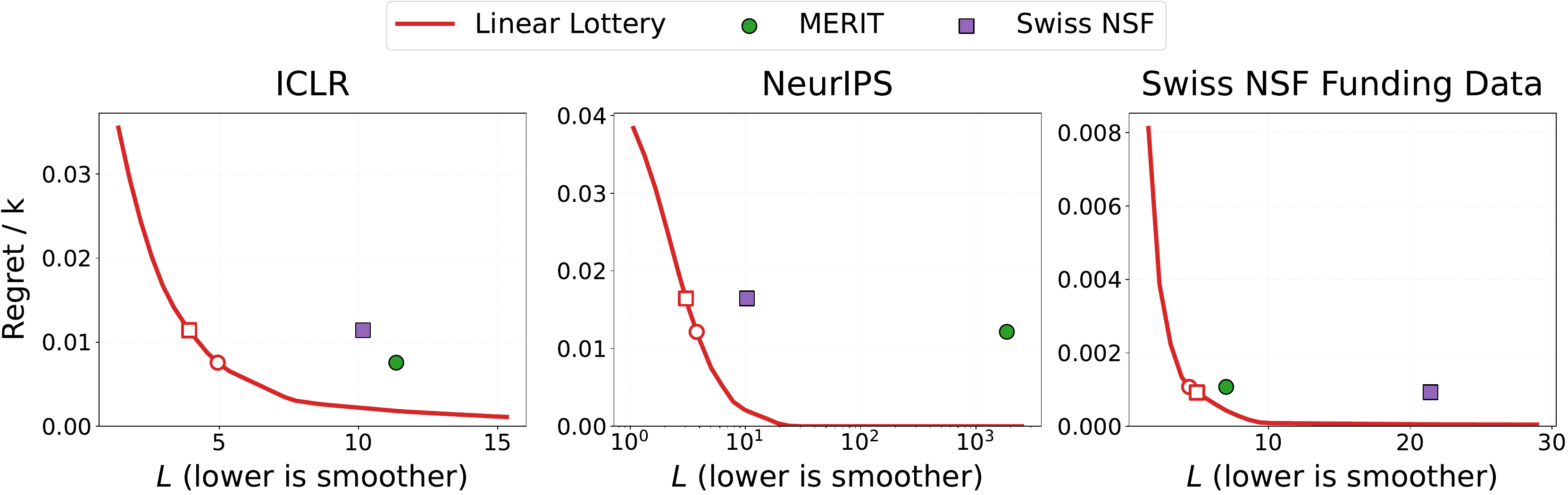}
    \caption{Regret--smoothness tradeoff at acceptance rate 50\%. Points show MERIT and Swiss NSF at their empirical local sensitivity under the worst one-review perturbation. Hollow red markers show worst-case smoothness guarantee of \ourlottery at the same regret level.}
    \label{fig:matched_regret_50}
\end{figure}

\clearpage

\end{document}

%% file: figures/diagram1.tex
\begingroup
\begin{tikzpicture}

\definecolor{beforecol}{RGB}{200,55,55}
\definecolor{aftercol}{RGB}{35,110,210}

\def\drawsmoothbarpair#1#2#3{%
    \pgfmathsetmacro{\xval}{#1}%
    \pgfmathsetmacro{\ybval}{#2}%
    \pgfmathsetmacro{\yaval}{#3}%
    \pgfmathsetmacro{\ymval}{min(\ybval,\yaval)}%
    \pgfmathsetmacro{\bwval}{0.26}%
    \pgfmathsetmacro{\gapval}{0.05}%
    \pgfmathsetmacro{\bLval}{\xval-\gapval-\bwval}%
    \pgfmathsetmacro{\bRval}{\xval-\gapval}%
    \pgfmathsetmacro{\aLval}{\xval+\gapval}%
    \pgfmathsetmacro{\aRval}{\xval+\gapval+\bwval}%
    \edef\tmp{%
        \noexpand\path[draw=beforecol!95!black, line width=0.45pt, fill=white,
            pattern=north east lines, pattern color=beforecol!95!black]
        (axis cs:\bLval,0) rectangle (axis cs:\bRval,\ymval);%
    }\tmp
    \edef\tmp{%
        \noexpand\path[draw=aftercol!95!black, line width=0.45pt, fill=white,
            pattern=north west lines, pattern color=aftercol!95!black]
        (axis cs:\aLval,0) rectangle (axis cs:\aRval,\ymval);%
    }\tmp
    \edef\tmp{%
        \noexpand\path[draw=beforecol!95!black, line width=0.45pt, fill=beforecol]
        (axis cs:\bLval,\ymval) rectangle (axis cs:\bRval,\ybval);%
    }\tmp
    \edef\tmp{%
        \noexpand\path[draw=aftercol!95!black, line width=0.45pt, fill=aftercol]
        (axis cs:\aLval,\ymval) rectangle (axis cs:\aRval,\yaval);%
    }\tmp
}

\begin{groupplot}[
    group style={
        group size=3 by 1,
        horizontal sep=1.95cm,        
    },
    width=0.34\textwidth,
    height=0.40\textwidth,
    xmin=0.45, xmax=8.55, ymin=0,
    xtick={1,...,8},
    xticklabels={1,2,3,4,5,6,7,8},
    tick style={black},
    xticklabel style={font=\scriptsize},
    yticklabel style={font=\scriptsize},
    axis lines*=left,
    axis line style={black},
    xlabel={Proposal index\\[-0.4ex]{\scriptsize(sorted by initial score)}},
    xlabel style={font=\small, align=center, yshift=1mm},
    ylabel style={font=\small, align=center, yshift=-1mm},
    clip=false,
    enlarge x limits=false,
]

\nextgroupplot[
    ylabel={Aggregated review score},
    ymax=1.0, ytick={0,0.2,0.4,0.6,0.8,1.0},
]
\foreach \i/\before/\after in {
    1/0.18/0.18, 2/0.26/0.26, 3/0.34/0.34, 4/0.42/0.42,
    5/0.47/0.53, 6/0.56/0.56, 7/0.68/0.68, 8/0.84/0.84}
    {\drawsmoothbarpair{\i}{\before}{\after}}
\draw[black, thick, rounded corners=1pt]
    (axis cs:4.93,0.43) rectangle (axis cs:5.43,0.57);
\node[black, font=\small, align=center]
    at (axis cs:3.4,0.82) {small increase\\in one review};
\draw[->, thick] (axis cs:3.9,0.73) -- (axis cs:4.95,0.55);

\nextgroupplot[
    ylabel={Marginal selection probability},
    ymax=1.05, ytick={0,0.25,0.50,0.75,1.00},
    ymajorgrids, grid style={dashed, gray!45},
]
\foreach \i/\before/\after in {
    1/0.00/0.00, 2/0.00/0.00, 3/0.00/0.00, 4/0.00/0.00,
    5/0.00/0.333, 6/0.50/0.333, 7/0.50/0.333, 8/1.00/1.00}
    {\drawsmoothbarpair{\i}{\before}{\after}}
\draw[black, thick, rounded corners=1pt]
    (axis cs:4.93,0.02) rectangle (axis cs:5.43,0.38);
\draw[black, thick, rounded corners=1pt]
    (axis cs:5.57,0.29) rectangle (axis cs:6.07,0.54);
\node[black, font=\small, align=center]
    at (axis cs:2.7,0.78) {large\\probability\\jump};
\draw[->, thick] (axis cs:3.4,0.66) -- (axis cs:4.93,0.30);
\draw[->, thick] (axis cs:3.5,0.74) -- (axis cs:5.55,0.46);

\nextgroupplot[
    ylabel={Marginal selection probability},
    ymax=1.05, ytick={0,0.25,0.50,0.75,1.00},
    ymajorgrids, grid style={dashed, gray!45},
]
\foreach \i/\before/\after in {
    1/0.00/0.00, 2/0.00/0.00, 3/0.00/0.00, 4/0.11/0.09,
    5/0.19/0.27, 6/0.34/0.32, 7/0.54/0.52, 8/0.81/0.79}
    {\drawsmoothbarpair{\i}{\before}{\after}}
\draw[black, thick, rounded corners=1pt]
    (axis cs:4.93,0.15) rectangle (axis cs:5.43,0.31);
\draw[black, thick, rounded corners=1pt]
    (axis cs:5.57,0.27) rectangle (axis cs:6.07,0.39);
\node[black, font=\small, align=center]
    at (axis cs:3.0,0.58) {small\\change};
\draw[->, thick] (axis cs:3.7,0.50) -- (axis cs:4.93,0.235);
\draw[->, thick] (axis cs:3.8,0.52) -- (axis cs:5.55,0.34);

\end{groupplot}

\node[anchor=north, font=\small\bfseries, align=center]
    at ($(group c1r1.south)+(0,-1.15cm)$) {(a) Aggregated\\review scores};
\node[anchor=north, font=\small\bfseries, align=center]
    at ($(group c2r1.south)+(0,-1.15cm)$) {(b) Thresholded\\partial lottery};
\node[anchor=north, font=\small\bfseries, align=center]
    at ($(group c3r1.south)+(0,-1.15cm)$) {(c) Clipped Linear Lottery\\(ours)};

\node[
    draw=gray!50,
    rounded corners=2pt,
    fill=white,
    inner xsep=10pt,
    inner ysep=5pt,
    anchor=south,
    font=\small,
    align=center,
] at ($(group c2r1.north)+(0,0.55cm)$) {%
    \tikz[baseline=-0.55ex]{%
        \path[draw=beforecol!95!black, line width=0.45pt, fill=white,
              pattern=north east lines, pattern color=beforecol!95!black]
              (0,-0.10) rectangle (0.36,0.18);%
    }~before \hspace{1.3em}%
    \tikz[baseline=-0.55ex]{%
        \path[draw=aftercol!95!black, line width=0.45pt, fill=white,
              pattern=north west lines, pattern color=aftercol!95!black]
              (0,-0.10) rectangle (0.36,0.18);%
    }~after \hspace{1.3em}%
    \tikz[baseline=-0.55ex]{%
        \path[draw=beforecol!95!black, line width=0.45pt, fill=beforecol]
              (0,-0.10) rectangle (0.18,0.18);
        \path[draw=aftercol!95!black, line width=0.45pt, fill=aftercol]
              (0.18,-0.10) rectangle (0.36,0.18);
    }~change%
};

\end{tikzpicture}
\endgroup

%% file: figures/linear_lottery_example.tex
\begin{figure}[tb]
\centering
\resizebox{0.9\linewidth}{!}{%
\begin{tikzpicture}[
    >=Latex,
    font=\footnotesize,
    box/.style={
        draw,
        rounded corners,
        align=center,
        minimum width=3.0cm,
        minimum height=1.05cm,
        inner sep=4pt
    },
    graybox/.style={box, fill=llgray},
    bluebox/.style={box, fill=llblue},
    orangebox/.style={box, fill=llorange},
    greenbox/.style={box, fill=llgreen}
]

\node[graybox] (u) at (0,0)
{\textbf{Utilities}\\ $\util(\reviewmatrix) = (0.1,\,0.4,\,0.7,\,1.0)$};

\node[bluebox] (s) at (3.8,0)
{\textbf{Scale by $\slope=2$}\\ $(0.2,\,0.8,\,1.4,\,2.0)$};

\node[orangebox] (b) at (7.6,0)
{\textbf{Shift by $\intercept=-0.6$}\\ $(-0.4,\,0.2,\,0.8,\,1.4)$};

\node[greenbox] (p) at (11.4,0)
{\textbf{Clip to $[0,1]$}\\ $\selectionprob(\reviewmatrix) = (0,\,0.2,\,0.8,\,1)$};

\draw[->, thick] (u.east) -- (s.west);
\draw[->, thick] (s.east) -- (b.west);
\draw[->, thick] (b.east) -- (p.west);

\end{tikzpicture}%
}
\caption{
A concrete example of the Clipped Linear Lottery with $\napps=4$ candidates and budget $\nselected=2$, following Algorithm~\ref{alg:linear_lottery}. Utilities are scaled by $\slope=2$, then shifted by an intercept $\intercept=-0.6$ chosen so that the clipped probabilities sum to the budget. Clipping the shifted values to $[0,1]$ gives marginals $\selectionprob(\reviewmatrix)=(0,0.2,0.8,1)$, which sum to $\nselected=2$. Thus, the lowest-utility candidate is auto-rejected, the highest is auto-accepted, and the two middle candidates form the active lottery pool.
}
\label{fig:linear_lottery_example}
\end{figure}

%% file: figures/baseline_local_sensitivity_table_existing.tex
\begin{table}
\centering
\small
\setlength{\tabcolsep}{4pt}
\begin{tabular}{l|rr|rr|rr}
\toprule
 & \multicolumn{2}{c|}{MERIT} & \multicolumn{2}{c|}{Swiss NSF}\\
Dataset & Max $\Delta p$ & Smoothness & Max $\Delta p$ & Smoothness \\
\midrule
ICLR      & 0.314 & 5.7 & 0.532 & 9.6  \\
NeurIPS   & 0.321 & 5.8 & 0.633 & 11.4  \\
Swiss NSF & 0.471 & 4.7 & 0.923 & 18.7  \\
Beta      & 0.500 & 9.0 & 0.667 & 24.0  \\
\bottomrule
\end{tabular}
\caption{Empirical evaluation of smoothness of existing partial lotteries under a single-review one-tick perturbation ($\nselected=10\%$). Max $\Delta p$ is the maximum coordinate change and Smoothness is local sensitivity ($\Delta p_{L1}$/input perturbation magnitude). Existing lotteries are highly non-smooth, as a one-point change in a single review induces large changes in selection probabilities across all datasets.}
\label{tab:baseline_local_sensitivity_existing_10}
\end{table}

%% file: figures/baseline_local_sensitivity_table_existing_appendix.tex
\begin{table}
\centering
\small
\setlength{\tabcolsep}{4pt}
\begin{tabular}{llrrrrrr}
\toprule
 & & \multicolumn{2}{c}{MERIT} & \multicolumn{2}{c}{Swiss NSF}  \\
Dataset & $k$ & Max $\Delta p$ & Smoothness & Max $\Delta p$ & Smoothness \\
\midrule
ICLR & 33\% & 0.001 & 5.4 & 0.634 & 11.4  \\
ICLR & 50\% & 0.504 & 11.4 & 0.565 & 10.2  \\
\midrule
NeurIPS & 33\% & 0.846 & 187.1 & 0.585 & 10.5  \\
NeurIPS & 50\% & 0.536 & 1852.9 & 0.572 & 10.3  \\
\midrule
Swiss NSF & 33\% & 0.571 & 10.0 & 0.524 & 5.2  \\
Swiss NSF & 50\% & 0.564 & 7.0 & 0.429 & 21.4  \\
\midrule
Beta & 33\% & 0.500 & 18.0 & 0.625 & 17.0  \\
Beta & 50\% & 0.367 & 7.0 & 0.556 & 28.0  \\
\bottomrule
\end{tabular}
\caption{Existing baseline partial lotteries, $k \in \{33\%, 50\%\}$. Local sensitivity under a single-review one-tick perturbation. Max $\Delta p$ is the maximum coordinate change and Smoothness is local sensitivity ($\Delta p_{L1}$/input perturbation magnitude).}
\label{tab:baseline_local_sensitivity_existing_3350}
\end{table}